\documentclass[twocolumn,10pt]{article}
\usepackage[twocolumn,textwidth=18.31cm,columnsep=0.81cm]{geometry}
\usepackage{relsize}
\usepackage[utf8x]{inputenc}

\usepackage{amsmath}
\usepackage{amssymb}
\usepackage{amsthm}
\usepackage{amsxtra}
\usepackage{bbm}
\usepackage{accents}
\usepackage{enumerate}
\usepackage{esint}
\usepackage[ruled]{algorithm2e}
\usepackage{tikz}
\usepackage{multirow,rotating}
\usepackage[ruled]{algorithm2e}

\usepackage{url}
\usepackage[nocompress]{cite}

\usepackage{graphicx}

\usepackage{booktabs}
\usepackage{color, colortbl}
\definecolor{Gray}{gray}{0.9}
\newcolumntype{g}{>{\columncolor{Gray}}r}
\usepackage{array}

\numberwithin{equation}{section}

\newtheorem{theorem}{Theorem}
\newtheorem{definition}{Definition}
\newtheorem{corollary}{Corollary}
\newtheorem{proposition}{Proposition}
\newtheorem{lemma}{Lemma}
\theoremstyle{remark}
\newtheorem{remark}{Remark}

\providecommand{\abs}[1]{\lvert#1\rvert}

\newcommand{\bitem}{\begin{itemize}}
\newcommand{\eitem}{\end{itemize}}

\newcommand{\R}{\mathbb{R}}

\newcommand{\V}{\mathbb{V}}

\newcommand{\SE}{\mathcal{E}}
\newcommand{\SV}{\mathcal{V}}
\newcommand{\SI}{\mathcal{I}}
\newcommand{\SM}{\mathcal{M}}

\newcommand{\bpm}{\begin{pmatrix}}
\newcommand{\epm}{\end{pmatrix}}
\newcommand{\bsm}{\left(\begin{smallmatrix}}
\newcommand{\esm}{\end{smallmatrix}\right)}

\newcommand{\lan}{\left\langle}
\newcommand{\ran}{\right\rangle}

\DeclareMathOperator{\argmin}{arg min}

\DeclareMathOperator{\conv}{conv}

\DeclareMathOperator{\nb}{nb}

\newcolumntype{M}{>{$\vcenter\bgroup\hbox\bgroup}c<{\egroup\egroup$}}

\definecolor{darkgreen}{rgb}{0,0.5,0}
\newcommand{\minus}{$-$}
\newcommand{\plus}{$+$}
\usepackage[pagebackref=true,breaklinks=true,letterpaper=true,colorlinks,bookmarks=false]{hyperref}

\usepackage[utf8x]{inputenc}

\usepackage{floatrow}

\usepackage{tikz}                 
\usetikzlibrary{patterns}
\usetikzlibrary{fit}
\usetikzlibrary{positioning}
\usetikzlibrary{decorations.markings}
\usepackage{pgfplots}             
\pgfdeclarelayer{background}
\pgfsetlayers{background,main}
\usepackage{sidecap}

\usepackage{stmaryrd}

\newcommand{\myparagraph}[1]{\noindent \textbf{#1}}
\newcommand{\addtag}{\refstepcounter{equation}\tag{\theequation}} 

\begin{document}

\title{Partial Optimality by Pruning for MAP-Inference\\ with General Graphical Models}

\author{Paul Swoboda,
	Alexander Shekhovtsov,
        J\"org~Hendrik~Kappes,
        Christoph~Schn\"orr,
        and~Bogdan~Savchynskyy
}
\date{}
%
%
%
%

\maketitle

\begin{abstract}
We consider the energy minimization problem for undirected graphical models, also known as MAP-inference problem for Markov random fields which is NP-hard in general.
We propose a novel polynomial time algorithm to obtain a part of its optimal {\em non-relaxed integral} solution. 
Our algorithm is initialized with variables taking integral values in the solution of a convex relaxation of the MAP-inference problem and iteratively prunes those, which do not satisfy our criterion for partial optimality. 
We show that our pruning strategy is in a certain sense theoretically optimal. Also empirically our method outperforms previous approaches in terms of the number of persistently labelled variables. 
The method is very general, as it is applicable to models with arbitrary factors of an arbitrary order and can employ any solver for the considered relaxed problem. 
Our method's runtime is determined by the runtime of the convex relaxation solver for the MAP-inference problem.
\end{abstract}



\section{Introduction}
Finding the most likely configuration of a Markov random field (MRF), also called MAP-inference or energy minimization problem for graphical models, is of big importance in computer vision, bioinformatics, communication theory, statistical physics, combinatorial optimization, signal processing, information retrieval and statistical machine learning, see~\cite{PIC2011,GraphicalModelsWainwrightJordan, Kappes2013Benchmark} for an overview of applications.
This key problem however is NP-hard. Therefore approximate methods have been developed to tackle big instances commonly arising in image processing, see~\cite{SzeliskiComparativeStudyMRF, Kappes2013Benchmark} for an overview of such methods. 
These approximate methods often cannot find an optimal configuration, but deliver close solutions. 
If one could prove, that some variables of the solution given by such approximate algorithms belong to an optimal configuration, the value of such approximate methods would be greatly enhanced.
In particular, the problem for the remaining variables could be solved by stronger, but computationally more expensive methods to obtain a global optimum as done e.g. in~\cite{PartialOptimalityAndILPKappes}.

In this paper we propose a way to gain such a partially optimal solution for the MAP-inference problem with \emph{general} discrete MRFs from possibly also non-exact solutions of the commonly used local polytope relaxation (see~\cite{ALinearProgrammingApproachToMaxSumWerner}). 
Solving over the local polytope amounts to solving a linear problem for which {\em any} linear programming (LP) solver can be used and for which dedicated and efficient algorithms exist.

\subsection{Related Work}\label{sec:related-work}

We distinguish two classes of approaches to partial optimality. 

\myparagraph{(i) Roof duality based approaches.} The earliest paper dealing with persistency is~\cite{NemhauserStableSetPartialOptimality}, which states a persistency
criterion for the stable set problem and verifies it for every solution of a certain relaxation. This relaxation is the same, as used by the roof duality method in~\cite{PseudoBooleanOptimizationBorosHammer} and which is also the basis
for the well known QPBO-algorithm~\cite{ExtendedRoofDuality,PseudoBooleanOptimizationBorosHammer}.
The MQPBO method~\cite{PartialOptimalityInMultiLabelMRFsKohli} extends roof duality to the multi-label case. 
The authors transform multi-label problems into quadratic binary ones and solve them via QPBO~\cite{PseudoBooleanOptimizationBorosHammer}.
However, their transformation is dependent upon choosing a label order and their results are so as well, see the experiments in~\cite{PartialOptimalityByPruningPotts}, where the label order is sampled randomly. It is not known how to choose an optimal label order to obtain the maximum number of persistent variables.

The roof duality method has been extended to higher order binary problems in~\cite{GeneralizedRoofDualityStrandmark,HigherOrderGraphCutFix,TransformationGeneralToFirstOrderIshikawa,GeneralizedRoofDualityBisubmodularKolmogorov}. 
The generalized roof duality method for binary higher order problems~\cite{GeneralizedRoofDualityStrandmark,GeneralizedRoofDualityBisubmodularKolmogorov} computes partially optimal variables directly for higher order potentials, while Ishikawa's and Fix et al's approaches~\cite{TransformationGeneralToFirstOrderIshikawa,HigherOrderGraphCutFix} transform the higher order problem to one with unary and pairwise terms only. 
Fix et al's method~\cite{HigherOrderGraphCutFix} is an improvement upon Ishikawa's~\cite{TransformationGeneralToFirstOrderIshikawa}.

Windheuser et al~\cite{GeneralizedRoofDualityWindheuser} proposed {\em a multi-label higher-order} roof duality method, which is a generalization of both MQPBO~\cite{PartialOptimalityInMultiLabelMRFsKohli} to higher order and Kahl and Strandmark's work~\cite{GeneralizedRoofDualityStrandmark} to the multi-label case.
However Windheuser et al neither describe an implementation nor provide experimental validation for the higher order multi-label case.

\myparagraph{(ii) Labeling testing approaches.} A different approach, specialized for Potts models, is pursued by Kovtun~\cite{KovtunPartialOptimalLabeling}, where
possible labelings are tested for persistency by auxiliary submodular problems. The parametric max-flow method for the Potts model by Gridchin and Kolmogorov~\cite{PottsModelParametricMaxFlowGridchyn} reduces the number of max-flow computations to compute the persistencies of Kovtun's method~\cite{KovtunPartialOptimalLabeling} to  $log(K)$, where $K$ is the number of labels. The dead-end elimination procedure~\cite{DeadEndEliminationDesmet} tests, if certain labels of nodes cannot belong to an optimal solution. It is a local heuristic and does not perform any optimization.

Since for non-binary multi-labeling problems the submodular approximations constructed by approaches of class (i) are provably less tight than the standard local polytope relaxation~\cite[Prop.~1]{MQPBOShekhovtsovTP}, we consider class (ii) in this paper. 
Specifically, based on ideas in~\cite{PartialOptimalityByPruningPotts} to handle the Potts model, we develop a theoretically substantiated approach to recognizing partial optimality for \emph{general} graphical models, together with a competitive comparison to the $5$ approaches~\cite{PartialOptimalityInMultiLabelMRFsKohli,KovtunPartialOptimalLabeling,GeneralizedRoofDualityStrandmark,HigherOrderGraphCutFix,TransformationGeneralToFirstOrderIshikawa} discussed above, that define the state-of-the-art.

\myparagraph{Unified study.} In addition we point to the recent paper~\cite{ShekhovtsovCVPR2014}, which provides a unified study of most mentioned methods and a systematic way of their analysis. While their persistency criterion is provably not weaker than ours, due to the general structure of the resulting LP it cannot be applied to large-scale problems in a straightforward manner. Moreover, our approach is directly applicable to higher order models and tighter then the local polytope relaxations, whereas~\cite{ShekhovtsovCVPR2014} requires generalization to these cases, though such a generalization is presumably possible. We show that our algorithm solves a special case of the maximal presistency problem formulated in~\cite{ShekhovtsovCVPR2014}.
\myparagraph{Shrinking technique.} The recent work~\cite{SavchynskyyNIPS2013} proposes a method for efficient shrinking of the combinatorial search area with the local polytope relaxation. Though the algorithmic idea is similar to the presented one, the method~\cite{SavchynskyyNIPS2013} does not provide partially optimal solutions. We refer to Section~\ref{sec:PersistencyAlgorithm} for further discussion.

Furthermore, preliminary shorter version of the our study was published at a conference as~\cite{PartialOptimalitySwobodaCVPR2014}.

\subsection{Contribution and Organization}
Adopting ideas from~\cite{PartialOptimalityByPruningPotts}, we propose {\em a novel} method for computing partial optimality, which is applicable to \emph{general graphical models with arbitrary higher order potentials}.
Similarly to~\cite{PartialOptimalityByPruningPotts} our algorithm is initialized with variables taking integral values in the solution of a convex relaxation of the MAP-inference problem and iteratively prunes those, which do not satisfy our persistency criterion. 
We show that our pruning strategy is in a certain sense theoretically optimal. 
Though the used relaxation can be chosen arbitrarily, for brevity we restrict our exposition and experiments to the local polytope relaxation. 
Tighter relaxations {\em provably} yield better results. 
However even by using the local polytope relaxation we can often achieve {\em a substantially higher} number of persistent variables, than competing approaches, which we confirm experimentally. 
We also show how our approach can be made invariant against reparametrizations. This improves our partial optimality criterion and we can show equivalence with the all-to-one improving mapping class of partial optimality methods proposed in~\cite{ShekhovtsovCVPR2014}.
Our approach is very general, as it can use {\em any}, also approximate, solver for the considered convex relaxation. 
Moreover, the computational complexity of our method is determined mainly by the runtime of the used solver.

The comparison to existing persistency methods is summarized in Table~\ref{tab:Comparison}.
\begin{table*}[htb]
\centering
\begin{tabular}{|l|c|c|c|c|} 
\toprule
Work & {non-binary} & {higher order} & {non-Potts} & Auxiliary problem \\
\hline
Boros \& Hammer 2002~\cite{PseudoBooleanOptimizationBorosHammer} & \minus & \minus & \plus & QPBO\\
Kovtun 2003\cite{KovtunPartialOptimalLabeling} & \plus & \minus & \minus & submodular binary\\
Rother et al. 2007~\cite{ExtendedRoofDuality} & \minus & \minus & \plus & QPBO\\
Kohli et al. 2008~\cite{PartialOptimalityInMultiLabelMRFsKohli} & \plus & \minus & \plus & QPBO\\
Kovtun 2011~\cite{Kovtun2011} & \plus & \minus & \plus & submodular multilabel\\
Ishikawa 2011~\cite{TransformationGeneralToFirstOrderIshikawa}& \minus & \plus & \plus & QPBO \\
Fix et al. 2011~\cite{HigherOrderGraphCutFix} & \minus & \plus & \plus & QPBO \\
Kolmogorov 2012~\cite{GeneralizedRoofDualityBisubmodularKolmogorov} & \minus & \plus & \plus & bi-submodular \\
Kahl \& Strandmark 2012~\cite{GeneralizedRoofDualityStrandmark} & \minus & \plus & \plus & bi-submodular + LP\\
Windheuser et al. 2012~\cite{GeneralizedRoofDualityWindheuser} & \plus & \plus & \plus & bi-submodular\\
Swoboda et al. 2013~\cite{PartialOptimalityByPruningPotts} & \plus & \minus & \minus & iterated local polytope\\
Shekhovtsov 2014~\cite{ShekhovtsovCVPR2014} & \plus & \minus & \plus & local polytope with extra variables\\
\hline
{\bf Ours} & {\bf \plus} & {\bf \plus} & {\bf \plus} & {\bf any convex relaxation}\\
\bottomrule
\end{tabular}
\caption{Comparison between partial optimality methods. Detailed descriptions are presented in Section~\ref{sec:related-work}.}
\label{tab:Comparison}
\end{table*}

Our code together with the experimental setup is available at~\url{http://paulswoboda.net}.

\myparagraph{Organization.} 
In Section~\ref{sec:Preliminaries} we review the energy minimization problem and the local polytope relaxation, in Section~\ref{sec:Persistency} our persistency criterion is presented. 
The corresponding algorithm and its theoretical analysis are presented in Sections~\ref{sec:PersistencyAlgorithm}, \ref{sec:LargestPersistentLabeling} and~\ref{sec:Reparamerisation} respectively.
Extensions to the higher order case and tighter relaxations are discussed in Section~\ref{sec:HigherOrder}.
Section~\ref{sec:Experiments} provides experimental validation of our approach and a comparison to the existing methods~\cite{PartialOptimalityInMultiLabelMRFsKohli,KovtunPartialOptimalLabeling,GeneralizedRoofDualityStrandmark,HigherOrderGraphCutFix,TransformationGeneralToFirstOrderIshikawa}.


\section{MAP-Inference Problem}
\label{sec:Preliminaries}
The MAP-inference problem for a graphical model over an undirected  graph $G=(\mathcal{V},\mathcal{E})$, reads 
\begin{equation}
 \label{eq:GraphicalModel}
\min_{x \in X_{\mathcal{V}}} E_{\mathcal{V}}(x) := 
\sum\limits_{v \in \mathcal{V}} \theta_v(x_v) + \sum\limits_{uv \in \mathcal{E}} \theta_{uv}(x_u,x_v)\,,
\end{equation}
where $x_u$ belongs to a finite \emph{label set} $X_u$ for each node $u \in \mathcal{V}$, 
$\theta_u: X_u \rightarrow \R$ and $\theta_{uv}: X_u \times X_v \rightarrow \R$ are the \emph{unary} and \emph{pairwise potentials} associated with the nodes and edges of $G$. The label space for $A \subset \mathcal{V}$ is $X_A = \bigotimes_{u \in A} X_u$, where $\bigotimes$ stands for the Cartesian product. 
For notational convenience we write $X_{uv} = X_u \times X_v$ and $x_{uv}=(x_u,x_v)$ for $uv \in \mathcal{E}$. 
Notations like $x \in X_{A}$ implicitly indicate that the vector $x$ only has components $x_{u}$
indexed by $u \in A$. With $x_{|A} \in X_A$ we denote restriction of the labeling $x\in X_{\SV}$ to the set $A\subset \SV$.

More general graphical models with terms depending on three or more variables can be considered as well.
For brevity we restrict ourselves here to the pairwise case. An extension to the higher order case is discussed in Section~\ref{sec:HigherOrder}.

Problem~\eqref{eq:GraphicalModel} is equivalent to the integer linear problem
{\small
\vspace*{-0.2cm}
\begin{align} 
\label{eq:MarginalPolytopProblem}
&\min_{\mu\in\Lambda_\mathcal{V}} \sum_{v\in \mathcal{V}} \sum_{x_v\in X_v}\hspace{-2pt}\theta_v(x_v)\mu_v(x_v) 
\vspace*{1pt} +\sum_{uv\in \mathcal{E}}\sum_{x_{uv}\in X_{uv}}\hspace{-6pt}\theta_{uv}(x_{uv})\mu_{uv}(x_{uv})
\nonumber\\
&\text{s.t. }\ \mu_w(x_w) \in \{0,1\}\ \mbox{for}\ w\in \mathcal{V}\cup \mathcal{E},\ x_w\in X_w\,,
\end{align}
}
where {\em the local polytope} $\Lambda_\mathcal{V}$~\cite{GraphicalModelsWainwrightJordan} is the set of $\mu$ fulfilling
{\small
\begin{equation}
\label{eq:LocalPolytope}
\begin{array}{l}
 \sum_{x_v \in \mathcal{V}}\mu_v(x_v)=1,\ v \in \mathcal{V},\\
 \sum_{x_v \in \mathcal{V}}\mu_{uv}(x_u,x_v)=\mu_u(x_u),\ x_u\in X_u,\ uv \in \mathcal{E},\\
 \sum_{x_u \in \mathcal{V}}\mu_{uv}(x_u,x_v)=\mu_v(x_v),\ x_v\in X_v,\ uv \in \mathcal{E},\\
 \mu_w(x_w)\ge 0,\ w\in \mathcal{V}\cup \mathcal{E},\ x_w\in X_w\,.
\end{array}
\end{equation}
}
We define $\Lambda_A$ for $A \subset \mathcal{V}$ similarly.
Slightly abusing notation we will denote the objective function in~\eqref{eq:MarginalPolytopProblem} as $E_{\mathcal{V}}(\mu)$.
The formulation~\eqref{eq:MarginalPolytopProblem} utilizes the overcomplete representation~\cite{GraphicalModelsWainwrightJordan} of labelings in terms of indicator vectors~$\mu$, which are often called {\em marginals}. The problem of finding ${\mu^*\in \argmin_{\mu\in\Lambda_\mathcal{V}}E_{\mathcal{V}}(\mu)}$ (i.e. solving~\eqref{eq:MarginalPolytopProblem} without integrality constraints) is called {\em the local polytope relaxation} of~\eqref{eq:GraphicalModel}.

While solving the local polytope relaxation can be done in polynomial time, the corresponding optimal marginal $\mu^*$ may not be integral anymore, hence infeasible and not optimal for~\eqref{eq:MarginalPolytopProblem}.
For a wide spectrum of problems however most of the entries of optimal marginals $\mu^*$ for the local polytope relaxation will be integral. 
Unfortunately, there is no guarantee that any of these integral variables will be part of a globally optimal solution to~\eqref{eq:MarginalPolytopProblem}, except in the case of binary variables, that is $X_u = \{0,1\}$ $\forall u \in \mathcal{V}$, and unary and pairwise potentials~\cite{RoofDualityInQuadratic01OptimizationHammerHansenSimeone}.
Natural questions are: 
(i) Is there a subset $A\subset \mathcal{V}$ and a minimizer $\mu^0$ of the original NP-hard problem~\eqref{eq:MarginalPolytopProblem} such that $\mu^0_{v} = \mu^*_{v}$  $\forall v \in A$? 
In other words, is $\mu^*$ {\em partially optimal} or {\em persistent} on some set~$A$?
(ii) Given a relaxed solution $\mu^* \in \Lambda_\mathcal{V}$, how can we determine such a set $A$?
We provide a novel approach to tackle these problems in what follows.

\section{Persistency}
\label{sec:Persistency}
Assume we have marginals $\mu \in \Lambda_{\SV}$. 
We say that the marginal $\mu_{u}$, $u \in \SV$, is {\em integral} if 
$ \mu_{u}(x_u) \in \{0,1\}$ $\forall x_u \in X_u$.
In this case the marginal corresponds uniquely to a label $x_u$ with $\mu_{u}(x_u) = 1$. If this {\em integrality condition} holds for all $u\in\SV$ the corresponding vector $\mu$ will be denoted as $\delta(x)$. The convex hull of marginals corresponding to all labelings known as {\em marginal polytope} will be denoted as $\SM_{\SV}:=\conv(\delta(X_{\SV}))$. The non-relaxed energy minimization~\eqref{eq:GraphicalModel} can be equivalently written as $\min_{\mu\in\SM_{\SV}}E_{\SV}(\mu)$.

Let the boundary nodes and edges of a subset of nodes $A \subset \mathcal{V}$ be defined as follows:

\begin{definition}[Boundary and Interior]
\label{def:Boundary}
For the set $A\subset \mathcal{V}$ the set 
$\partial \mathcal{V}_A := {\{ u \in A \ \colon\ \exists v \in \mathcal{V}\backslash A \text{ s.t. } uv \in \mathcal{E} \}}$ is called {\em its boundary}.
The respective {\em set of boundary edges} is defined as 
$\partial \mathcal{E}_A = \{ uv \in \mathcal{E} \ \colon\ u \in A \textrm{ and } v \in \mathcal{V}\backslash A \}$.
The set $A \backslash \partial \mathcal{V}_A$ is called the interior of $A$.
\end{definition}

An example graph illustrating the concept of interior and boundary nodes can be seen in Figure~\ref{fig:ExampleGraph}.

\begin{figure}
\parbox{0.35\linewidth}{
\begin{tikzpicture}[scale=0.29]
\tikzstyle{outside}=[circle,thick,draw=blue!75,fill=red!50]
\tikzstyle{inside}=[circle,thick,draw=yellow!75,fill=yellow!20,pattern=crosshatch]
\tikzstyle{boundary}=[circle,thick,draw=green!75,fill=green!20,pattern=north west lines]

  \node[outside] (11) at (1,1) {};
  \node[outside] (12) at (1,3)  {};
  \node[outside] (13) at (1,5)  {};
  \node[outside] (21) at (3,1) {};
  \node[boundary] (22) at (3,3)  {};
  \node[boundary] (23) at (3,5)  {};
  \node[outside] (31) at (5,1) {};
  \node[boundary] (32) at (5,3)  {};
  
  \node[inside] (33) at (5,5)  {};

  \node[outside] (14) at (1,7) {};
  \node[outside] (15) at (1,9)  {};
  \node[boundary] (24) at (3,7) {};
  \node[outside] (25) at (3,9)  {};
  \node[boundary] (34) at (5,7) {};
  \node[outside] (35) at (5,9)  {};

  \node[outside] (41) at (7,1) {};
  \node[outside] (51) at (9,1) {};
  \node[boundary] (42) at (7,3) {};
  \node[outside] (52) at (9,3) {};
  \node[boundary] (43) at (7,5) {};
  \node[outside] (53) at (9,5) {};
  \node[outside] (44) at (7,7) {};
  \node[outside] (54) at (9,7) {};
  \node[outside] (45) at (7,9) {};
  \node[outside] (55) at (9,9) {};
  \foreach \from/\to in
{
11/12,12/13,21/22,22/23,31/32,32/33,13/23,23/33,12/22,22/32,11/21,21/31,
31/41,41/51,32/42,42/52,33/43,43/53,
14/24,24/34,34/44,44/54,15/25,25/35,35/45,45/55,
15/14,14/13,25/24,24/23,35/34,34/33,45/44,44/43,43/42,42/41,
55/54,54/53,53/52,52/51}
    \draw (\from) -- (\to);

\draw[blue,dashed] (2,2) -- (8,2);
\draw[blue,dashed] (8,2) -- (8,6);
\draw[blue,dashed] (8,6) -- (6,6);
\draw[blue,dashed] (6,6) -- (6,8);
\draw[blue,dashed] (6,8) -- (2,8);
\draw[blue,dashed] (2,8) -- (2,2);
\caption{}\label{fig:ExampleGraph}
\end{tikzpicture}
}
\parbox{0.02\linewidth}{\ }
\parbox{0.6\linewidth}{
\small Figure~1. An exemple graph containing inside nodes (yellow with crosshatch pattern) and boundary nodes (green with diagonal pattern). 
The blue dashed line encloses the set $A$. Boundary edges are those crossed by the dashed line. 
}
\vspace{-0.7cm}
\end{figure}

\begin{definition}[Persistency]\label{def:Persistency}
A labeling $x^0 \in X_{A}$ on a subset $A \subset \mathcal{V}$ is {\em partially optimal} or {\em persistent} if $x^0$ coincides with an optimal solution to~\eqref{eq:GraphicalModel} on $A$. 
\end{definition}
 
In the remainder of this section, we state our novel persistency criterion in Theorem~\ref{thm:PersistencyCriterion}. Taking additionally into account a convex relaxation yields a computationally tractable approach in Corollary~\ref{cor:RelaxedPersistencyCriterion}.

As a starting point, consider the following sufficient criterion for persistency of $x^0\in X_{A}$.
Introducing a {\em concatenation} of labelings $x^0\in X_A$ and $\tilde{x}\in X_{\mathcal{V}\backslash A}$ as 
$(x^0,\tilde x):=
\left\{
\begin{array}{rl}
 x^0_v, & v\in A,\\
 \tilde{x}_v, & v \in \mathcal{V}\backslash A
\end{array}
\right.
$,
the criterion reads: 
\begin{proposition}
\label{prop:BestPersistencyCriterion}
A partial labeling $x^0 \in X_A$ is persistent if there holds
\begin{equation}
\label{eq:BestPersistencyCriterion}
\forall \tilde{x} \in X_{\mathcal{V}\backslash A}  \ \colon\ 
x^0 \in \operatornamewithlimits{argmin}_{\substack{x \in X_{A}}} E_{\mathcal{V}}((x,\tilde{x}))\,.
\end{equation} 
\end{proposition}
\begin{proof}
Consider the equation
\begin{equation}
\label{eq:EnergyMinimumSplit}
 \min_{x\in X_{\mathcal{V}}} E_{\mathcal{V}}(x) = 
\min_{\tilde{x} \in X_{\mathcal{V}\backslash A}} \min_{\substack{x \in X_{A}}} E_{\mathcal{V}}((x,\tilde{x}))\,.
\end{equation}
Let $\tilde{x} \in X_{\mathcal{V}\backslash A}$ be such that it leads to a minimal value on the right hand side of~\eqref{eq:EnergyMinimumSplit}. 
Then $\tilde{x}$ is part of an optimal solution. By the assumption~\eqref{eq:BestPersistencyCriterion}, $x^0$ is an optimal solution to the inner minimization problem of~\eqref{eq:EnergyMinimumSplit}, hence $(x^0,\tilde{x})$ is optimal for~\eqref{eq:GraphicalModel}.
\end{proof}

This means that the concatenated labeling $(x^0,\tilde x)$ has to be optimal for 
$\min_{x} E(x) \text{ s.t. } x_v = \tilde{x}_v \forall v \in \mathcal{V} \backslash A$.
Informally this means that the solution $x^0$ is independent of what happens on $\mathcal{V}\backslash A$.
This criterion however is hard to check directly, as it entails solving NP-hard minimization problems over an exponential number of labelings $\tilde{x} \in X_{\mathcal{V}\backslash A}$.

We relax the above criterion~\eqref{eq:BestPersistencyCriterion} so that we have to check the solution of only \emph{one} energy minimization problem by modifying the unaries $\theta_v$ on boundary nodes so that they bound the influence of \emph{all} labelings on  $\mathcal{V}\backslash A$ uniformly.

\begin{definition}[Boundary potentials and energies]
\label{def:InsidePotentialAndFunctional}
For a set $A \subset \mathcal{V}$ and {\em a test} labeling $y \in X_{\partial \mathcal{V}_A}$, we define for each boundary edge $uv \in \partial \mathcal{E}_A$, $u \in \partial \mathcal{V}_A$ the ``boundary'' potential $\hat{\theta}_{uv,y_u}: X_u \rightarrow \R$ as 
\begin{equation}
\label{eq:InnerPotential}
\hat{\theta}_{uv,y_u}(x_u) := \left\{ 
\begin{array}{ll}
\max_{x_v \in X_v} \theta_{uv}(x_u,x_v), & y_u = x_u\\
\min_{x_v \in X_v} \theta_{uv}(x_u,x_v), & y_u \neq x_u\\
\end{array}\,.
\right.
\end{equation}
Define the energy $\hat{E}_{A,y}\colon X_A \rightarrow \R$ with test labeling $y$ as
\begin{equation}
\label{eq:InnerFunctional}
 \hat{E}_{A,y}(x) := E_A(x) \hspace{2pt} + \hspace{-12pt}
\sum\limits_{uv \in \partial \mathcal{E}_A\colon u \in \partial \mathcal{V}_A} \hat{\theta}_{uv,y_u}(x_u)\,,
\end{equation}
where 
$E_A(x) = \sum\limits_{u \in A} \theta_u(x_u) +\hspace{-8pt} \sum\limits_{uv \in \mathcal{E}: u,v\in A}\hspace{-5pt} \theta_{uv}(x_{uv})$ 
is the energy with potentials with support in~$A$.
\end{definition}

Given a test labeling $y \in X_A$, energy~\eqref{eq:InnerFunctional} assigns a higher value than the original energy~\eqref{eq:GraphicalModel} for all labelings conforming to~$y$ and makes it more favourable for all labelings to not conform to~$y$.
An illustration of a boundary potential is depicted by Figure~\ref{fig:InsidePotential}.

\begin{figure}
\floatbox{figure}
{
\caption{
Illustration of a boundary potential~$\hat{\theta}_{x^0}$ constructed in~\eqref{eq:InnerPotential}. 
The second label comes from the test labeling~$x^0$, therefore entries are maximized for the first row and minimized otherwise.
}
}
{
\begin{tikzpicture}

\node[circle, fill=blue] (x11) at (1,1) {};
\node[circle, fill=blue] (x12) at (1,3) {};

\node[circle, fill=blue] (x21) at (5,1) {};
\node[circle, fill=blue] (x22) at (5,2) {};
\node[circle, fill=blue] (x23) at (5,3) {};

\draw[thick,-] (x11) -- node[near start,below] (1121) {} (x21); 
\draw[thick,-] (x11) to (x22); 
\draw[thick,-] (x11) -- node[near start,above] (1123) {} (x23); 
\draw[thick,-] (x12) -- node[near start,below] (1221) {} (x21); 
\draw[thick,-] (x12) to (x22); 
\draw[thick,-] (x12) -- node[near start,above] (1223) {} (x23); 

\draw[thick,red!75] ([shift=(-10:1.5cm)]1,1) arc (-10:35:1.5cm);
\node[red!75] at (2.5,0.5) {min};
\draw[thick,green!75] ([shift=(-35:1.5cm)]1,3) arc (-35:10:1.5cm);
\node[green!75] (max) at (2.5,3.5) {max};


\node[left of=x11] {$x_u$};
\node[left of=x12] {$x_u^0$};
\node[right of=x22] {$x_v$};
\node[below=0.3cm of x11] {$u$};
\node[below=0.3cm of x21] {$v$};
\node[above=0.1cm of max] {$\hat{\theta}_{uv,x^0_u}(x_u)$};

\begin{pgfonlayer}{background}
\node[pattern=north west lines, draw=green!75, thick, fit=(x11) (x12), minimum width=1cm] (leftPot) {};
\node[fill=red!50, draw=blue!75, thick, fit=(x21) (x22) (x23), minimum width=1cm] (rightPot) {};
\end{pgfonlayer}
\end{tikzpicture}

%
%
%
%
%
%
}
\end{figure}
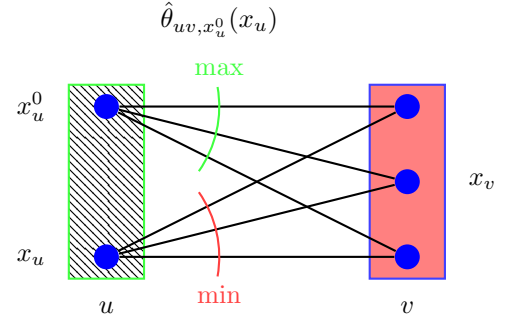
As a consequence, if the test labeling $y$ from Definition~\ref{def:Boundary} minimizes the energy~\eqref{eq:InnerFunctional}, the proof of the following theorem asserts that 
changing an arbitrary labeling $x \in X_{\mathcal{V}}$ as follows:
$x'_v = \left\{ \begin{array}{ll} y_v, &v \in A\\ x_v, &v \notin A\end{array} \right.$
will always result in a labeling with not bigger energy~\eqref{eq:GraphicalModel}, hence $y$ in particular fulfills the conditions~\eqref{eq:BestPersistencyCriterion} of Proposition~\ref{prop:BestPersistencyCriterion} and thus is persistent.
\begin{theorem}[Partial optimality criterion]
\label{thm:PersistencyCriterion}
A labeling $x^0 \in X_A$ on a subset $A \subset \mathcal{V}$ is persistent if
\begin{equation}
\label{eq:PersistencyMinimizationProblem}
x^0 \in \argmin_{x \in X_A} \hat{E}_{A,x^0}(x)\,,
\end{equation}
where $\hat{E}_{A,x^0}$ is the augmented energy functional~\eqref{eq:InnerFunctional}.
\end{theorem}
To prove the theorem we need the following technical lemma.

\begin{lemma}
\label{lemma:RestrictedPotentialInequality}
Let $A\subset \mathcal{V}$ be given together with ${y \in X_{\partial \mathcal{V}_A}}$. Let $x^0$ and $x'$ be two labelings on $\mathcal{V}$ such that $x^0|_A = y$. Then it holds for $uv \in \partial \mathcal{E}_A$, $u \in \partial \mathcal{V}_A$ that
\begin{equation}
\label{eq:RestrictedPotentialInequality1}
\theta_{uv}(x^0_u,x_v') + \hat{\theta}_{uv,y}(x_u') - \hat{\theta}_{uv,y}(x_u^0) \leq \theta_{uv}(x'_u,x'_v)\,.\hspace{-3pt}
\end{equation}
\end{lemma}
\begin{proof}
The case $x'_u = x_u^0$ is trivial. Otherwise, by Definition~\ref{def:InsidePotentialAndFunctional}, inequality~\eqref{eq:RestrictedPotentialInequality1} is equivalent to
\begin{multline}
\label{eq:RestrictedPotentialInequality2}
  \theta_{uv}(x_u^0,x_v') + \min_{x_v \in X_v} \theta_{uv}(x'_u,x_v) \\- \max_{x_v \in X_v} \theta_{uv}(x_u^0,x_v) - \theta_{uv}(x'_u,x'_v)
 \leq 0\,.
\end{multline}
Choose $x'_v$ for $x_v$ in the minimization and maximization in~\eqref{eq:RestrictedPotentialInequality2} to obtain the result.
\end{proof}

\begin{proof}[Proof of Theorem~\ref{thm:PersistencyCriterion}]
Let
\begin{equation}
\label{eq:OuterProblemWithBoundaryConditions}
\tilde{x} \in \arg\min\limits_{x \in X_{\SV}\atop x|_A=x^0|_A}E_{\SV}(x)\,.
\end{equation}
and let $x' \in X_{\SV}$ be an arbitrary labeling.
Then
{\allowdisplaybreaks
\begin{align*}
     & E_{\SV}(\tilde{x})= E_A(x^0) + E_{\mathcal{V}\backslash A}(\tilde{x}) + \hspace{-7pt}\sum_{uv \in \partial \mathcal{E}_A}\hspace{-5pt} \theta_{uv}(x^0_u,\tilde{x}_v)  \addtag \label{eq:PersistencyCriterionEq1}\\
=    & E_A(x^0) +  \sum_{uv \in \partial \mathcal{E}_A}  \hat{\theta}_{uv,y}(x^0_u)\\
     & + E_{\mathcal{V}\backslash A}(\tilde{x}) + \sum_{uv \in \partial \mathcal{E}_A} \left[ \theta_{uv}(x^0_u,\tilde{x}_v) - \hat{\theta}_{uv,y}(x^0_u) \right] \\
=    & \hat{E}_{A,x^0}(x^0) + E_{\mathcal{V}\backslash A}(\tilde{x}) + \hspace{-10pt}\sum_{uv \in \partial \mathcal{E}_A}\hspace{-7pt} \left[ \theta_{uv}(x^0,\tilde{x}_v) - \hat{\theta}_{uv,x^0}(x^0_u) \right] \\
\leq & \hat{E}_{A,x^0}(x') + E_{\mathcal{V}\backslash A}(x') +\hspace{-10pt} \sum_{uv \in \partial \mathcal{E}_A}\hspace{-7pt} \left[ \theta_{uv}(x^0,x'_v) - \hat{\theta}_{uv,x^0}(x^0_u) \right] \addtag \label{eq:PersistencyCriterionLeq1}\\
=    & E_A(x') + \sum_{uv \in \partial \mathcal{E}_A}  \hat{\theta}_{uv,x^0}(x'_u) \\
     & + E_{\mathcal{V}\backslash A}(x') + \sum_{uv \in \partial \mathcal{E}_A} \left[ \theta_{uv}(x^0_u,x'_v) - \hat{\theta}_{uv,x^0}(x^0_u) \right] \\
\leq & E_A(x')\hspace{-1.5pt} + \hspace{-1.5pt}E_{\mathcal{V}\backslash A}(x') \hspace{-1.5pt}+ \hspace{-10pt}\sum_{uv \in \partial \mathcal{E}_A}\hspace{-7pt} \theta_{uv}(x'_u,x'_v) \addtag \label{eq:PersistencyCriterionLeq2}
=  E_{\SV}(x').
\end{align*}
}
The equality~\eqref{eq:PersistencyCriterionEq1} is due to definition of $\tilde{x}$ in~\eqref{eq:RestrictedPotentialInequality2}.
The first inequality~\eqref{eq:PersistencyCriterionLeq1} is due to $x^0 \in \argmin_{x} \hat{E}_{A,x^0}(x)$, as assumed,
and 
of $\tilde{x}$ for~\eqref{eq:OuterProblemWithBoundaryConditions}.
The second inequality~\eqref{eq:PersistencyCriterionLeq2} is due to Lemma~\ref{lemma:RestrictedPotentialInequality}.
Hence $x^0$ is part of a globally optimal solution, as $x'$ was arbitrary.
\end{proof}

Checking the criterion in Theorem~\ref{thm:PersistencyCriterion} is NP-hard, because~\eqref{eq:PersistencyMinimizationProblem} is a MAP-inference problem of the same class as~\eqref{eq:GraphicalModel}. By relaxing the minimization problem~\eqref{eq:PersistencyMinimizationProblem} one obtains the polynomially verifiable persistency criterion in Corollary~\ref{cor:RelaxedPersistencyCriterion}.


\begin{corollary}[Tractable partial optimality criterion]
\label{cor:RelaxedPersistencyCriterion}
Labeling $x^0\in X_A$ on $A\subset\SV$ fulfilling the condition
\begin{equation}
\label{eq:TractablePartialOptimalityCriterion}
\delta(x^0) \in \argmin_{\mu\in\Lambda_A} \hat{E}_{A,x^0}(\mu)
\end{equation} 
is also a solution to~\eqref{eq:PersistencyMinimizationProblem}, hence persistent on $A$.
\end{corollary}

\begin{proof}
Expression~\eqref{eq:TractablePartialOptimalityCriterion} implies 
\begin{equation}
\delta(x^0) \in \argmin_{\mu\in\Lambda_A, \mu \in \{0,1\}^{\dim{\Lambda_A}}} \hat{E}_{A,x^0}(\mu)
\end{equation}
because $\delta(x^0)$ is integral by definition.
As~\eqref{eq:GraphicalModel} and~\eqref{eq:MarginalPolytopProblem} are equivalent and the corresponding labeling $x^0$ satisfies the conditions of Theorem~\ref{thm:PersistencyCriterion}, $x^0$ is partially optimal on $A$.
\end{proof}

\section{Persistency Algorithm}
\label{sec:PersistencyAlgorithm}
Now we concentrate on finding a set $A$ and labeling $x \in X_{A}$ such that the solution of
$\min_{\mu\in \Lambda_A} \hat{E}_{A,x}(\mu)$ fulfills the conditions of Corollary~\ref{cor:RelaxedPersistencyCriterion}. 
Our approach is summarized in Algorithm~\ref{alg:FindPersistency}. 

In the initialization step of Algorithm~\ref{alg:FindPersistency} we solve the relaxed problem over $\mathcal{V}$ without boundary labeling and initialize the set $A^0$ with nodes having an integer label.
Then in each iteration $t$ we minimize over the local polytope the energy $\hat{E}_{A^t,x^t}$ defined in~\eqref{eq:InnerFunctional}, corresponding to the set $A^t$ and boundary labeling coming from the solution of the last iteration.
We remove from $A^{t}$ all variables which are not integral or do not conform to the boundary labeling.
In each iteration $t$ of Algorithm~\ref{alg:FindPersistency} we shrink the set $A^t$ by removing variables taking non-integral values or not conforming to the current boundary condition.

\begin{algorithm}[t] \label{alg:FindPersistency}
 \KwData{$G=(\mathcal{V},\mathcal{E})$, $\theta_u: X_u \rightarrow \R$, $\theta_{uv}: X_{uv} \rightarrow \R$} 
 \KwResult{$A^*\subset \mathcal{V}$, $x^* \in X_{A^*}$}
 Initialize: \\
  Choose $\mu^{0} \in \argmin_{\mu \in \Lambda_\mathcal{V}} E_{\mathcal{V}}(\mu)$\\
  $A^0=\{u \in \mathcal{V}\colon \mu^0_u\in\{0,1\}^{|X_u|}\}$\\
  $t = 0$\\
\Repeat{$A^t = A^{t-1}$}{
  Set $x^t_u$ such that $\mu^{t}_{u}(x_u^t) = 1,\ u \in A^t$\\
  Choose $\mu^{t+1} \in \argmin_{\mu \in \Lambda_{A^{t}}} \hat{E}_{A^{t},x^{t}}(\mu)$\\
  $t = t+1$\\
  $W^t = \{ u \in \partial \mathcal{V}_{A^{t-1}}   \ \colon \ \mu^{t}_{u}(x^{t-1}_{u}) \neq 1 \}$\\
  $A^t = \{ u\in A^{t-1}\colon \mu^{t}_{u} \in \{0,1\}^{|X_u|} \} \backslash W^t$\\
}
$A^* = A^t$\\
Set $x^* \in X_{A^*}$ such that $\mu^{t}_u(x^*_u) = 1$\\
\caption{Finding persistent variables.}
\end{algorithm}

\myparagraph{Convergence.}
Since $\mathcal{V}$ is finite and $\abs{A^t}$ is monotonically decreasing, the algorithm converges in at most $\abs{\mathcal{V}}$ steps. 
Solving each subproblem in Algorithm~\ref{alg:FindPersistency} can be done in polynomial time. As the number of iterations of Algorithm~\ref{alg:FindPersistency} is at most $\abs{\mathcal{V}}$, Algorithm~\ref{alg:FindPersistency} itself is polynomial as well.
In practice only few iterations are needed.

After termination of Algorithm~\ref{alg:FindPersistency}, we have  
\begin{equation} 
\delta(x^*) \in \argmin_{\mu \in \Lambda_{A^*}} \hat{E}_{A^*,x^*}(\mu)\,.
\end{equation}
Hence $x^*$ and $A^*$ fulfill the conditions of  Corollary \ref{cor:RelaxedPersistencyCriterion}, which proves persistency.

\begin{figure*}
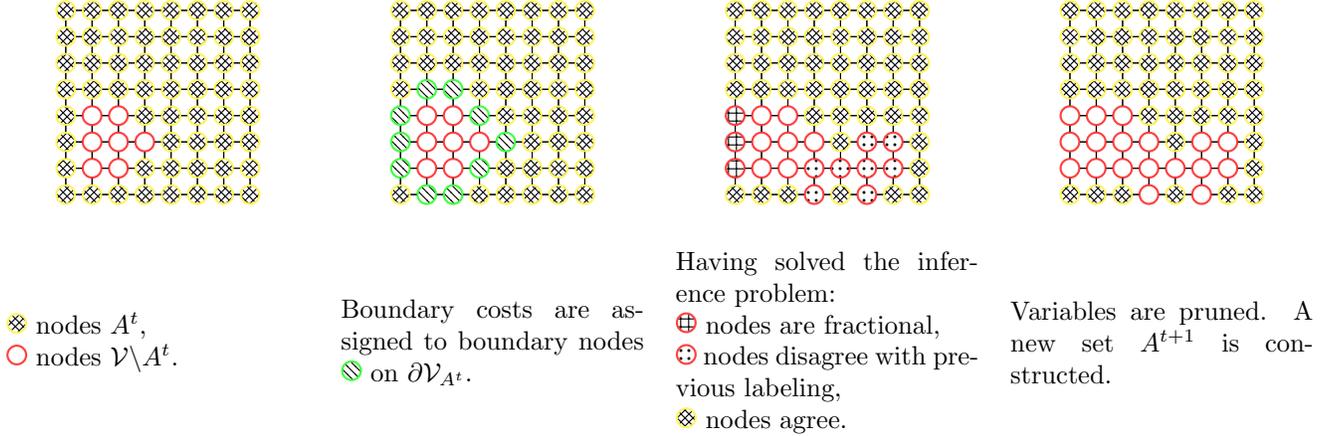

\begin{tabular}{cccc}

\begin{tikzpicture}[scale=0.35]
\input{./data_artificial/style3.tex}
\input{./data_artificial/initial_labeling}
\end{tikzpicture}
&
\begin{tikzpicture}[scale=0.35]
\input{./data_artificial/style3.tex}
\input{./data_artificial/boundary_conditions}
\end{tikzpicture}
&
\begin{tikzpicture}[scale=0.35]
\input{./data_artificial/style3.tex}
\input{./data_artificial/subsequent_labeling}
\end{tikzpicture}
&
\begin{tikzpicture}[scale=0.35]
\input{./data_artificial/style3.tex}
\input{./data_artificial/initial_labeling_next_iteration}
\end{tikzpicture}
\\

\vspace{0.07cm} & & &\\

\parbox{0.22\textwidth}{
\begin{tikzpicture}[scale=0.35]\input{./data_artificial/style3.tex} \node[i3] () at (0,0)  {}; \end{tikzpicture}
nodes $ A^t$, 
\\
\begin{tikzpicture}[scale=0.35]\input{./data_artificial/style3.tex} \node[outside] () at (0,0)  {}; \end{tikzpicture}
nodes $\mathcal{V} \backslash A^t$.
}
&
\parbox{0.22\textwidth}{
Boundary costs are assigned to boundary nodes
\begin{tikzpicture}[scale=0.35]\input{./data_artificial/style3.tex} \node[b255] () at (0,0)  {}; \end{tikzpicture}
on $\partial \mathcal{V}_{A^t}$.
}
&
\parbox{0.22\textwidth}{Having solved the inference problem:\\
\begin{tikzpicture}[scale=0.35]\input{./data_artificial/style3.tex} \node[fractional] () at (0,0)  {}; \end{tikzpicture}
nodes are fractional,\\
\begin{tikzpicture}[scale=0.35]\input{./data_artificial/style3.tex} \node[disagreeing] () at (0,0)  {}; \end{tikzpicture}
nodes disagree with previous labeling,\\
\begin{tikzpicture}[scale=0.35]\input{./data_artificial/style3.tex} \node[i3] () at (0,0)  {}; \end{tikzpicture}
nodes agree.
}
&
\parbox{0.22\textwidth}{Variables are pruned. A new set $A^{t+1}$ is constructed.}
\\

\end{tabular}
\centering
\caption{
Illustration of one iteration of Algorithm~\ref{alg:FindPersistency}. 
}
\end{figure*}

\myparagraph{Choice of Solver.}
All our results are independent of the specific algorithm one uses to solve the relaxed problems $\min_{\mu \in \Lambda_A}\hat{E}_{A,y}$, provided it returns an exact solution. However this can be an issue for large-scale datasets, where classical exact LP solvers like e.g.\ the simplex method become inapplicable. It is important that one can also employ {\em approximate} solvers, as soon as they provide (i) a proposal for {\em potentially} persistent nodes and (ii) sufficient conditions for optimality of the found {\em integral} solutions such as e.g.\ zero duality gap. These properties have the following precise formulation.
 
%

\begin{definition}[Integrally Correct Algorithm]
\label{def:ConsistencyMapping}
Assume an algorithm that takes as the input an energy minimization prboblem and outputs a labeling $x^* \in \bigotimes_{v \in \mathcal{V}}(\mathcal{X}_v\cup\{\#\})$.
We call such an algorithm {\em integrally correct} if $x^*_v\in X_v \forall v\in V$ implies 
$x^* \in \argmin\limits_{x \in X_{\mathcal{V}}} E_{\mathcal{V}}(x)$.

\end{definition}

Integrally correct algorithms include
\begin{itemize}
\item Dual decomposition based algorithms~\cite{
journals/pami/KomodakisPT11,
TRWSKolmogorov,
MAPBundleApproach,
AdaptiveDiminishingSmoothing,
Savchynskyy11
} 
deliver {\em strong tree agreement}~\cite{MapEstimationMessagePassingLPWainwright} and algorithms considering the Lagrangian 
dual~\cite{NormProductBPHazanShashua,FirstOrderPrimalDualMAPSchmidt,ConvergentMessagePassingGloberson}
return {\em strong arc consistency}~\cite{ALinearProgrammingApproachToMaxSumWerner} for some nodes. 
If one of these properties holds for a node $v$, we set $c_v$ as the corresponding label. Otherwise we set $c_v = \#$.
\item  Naturally, any algorithm solving $\min_{\mu \in \Lambda_\mathcal{V}} E(\mu)$ exactly is integrally correct with \\
$c_v=
\left\{
\begin{array}{rl}
 x_v, & \mu_v(x_v)=1\\
 \#, & \mu_v\notin\{0,1\}^{|X_v|}\,.
\end{array}
\right.$
\end{itemize}

\begin{proposition}
\label{prop:PersistencyAlgorithmWithConsistentSolver}
Let operations $\mu\in\argmin(...)$ in Algorithm~\ref{alg:FindPersistency} be exchanged with
\[\forall v \in \mathcal{V}, x_v\in X_v,\ \mu_v(x_v):=
\left\{
\begin{array}{rl}
1, & c_v=x_v\\
0, & c_v\notin\{x_v,\#\},\\
1/|X_v|, & c_v=\#
\end{array}
\right.
\]
where $c$ are consistent labelings returned by an integrally correct algorithm applied to the corresponding minimization problems. Then the output labeling $x^*$ is persistent. 
\end{proposition}
\begin{proof}
At termination of Algorithm~\ref{alg:FindPersistency} we have obtained a subset of nodes $A^*$, a test labeling $y^* \in X_{\partial \mathcal{V}_A}$, a labeling $x^*$ equal to $y^*$ on $\partial \mathcal{V}_A$ and a consistency mapping $c_u = x_u^*$ for $u \in A^*$. Hence, by Definition~\ref{def:ConsistencyMapping}, $x^* \in \argmin_{x \in X_A} \hat{E}_{A^*,y^*}$ and $x^*$ fulfills the conditions of Theorem~\ref{thm:PersistencyCriterion}.
\end{proof}
\begin{remark}
Note that a bad or early stopped solver, i.e. one which rarely (or even never) returns an optimality certificate or solves a weak relaxation, will also work with Algorithm~\ref{alg:FindPersistency}. 
However it will find smaller (or even empty) partial optimal solutions.
\end{remark}

\myparagraph{Comparison to the Shrinking Technique (CombiLP)~\cite{SavchynskyyNIPS2013}.} The recently published approach~\cite{SavchynskyyNIPS2013}, similar to Algorithm~\ref{alg:FindPersistency}, describes how to shrink the combinatorial search area with the local polytope relaxation. 
However (i) Algorithm~\ref{alg:FindPersistency} solves a series of auxiliary problems on the subsets $A^t$ of integer labels, whereas the method~\cite{SavchynskyyNIPS2013} considers nodes, which got fractional labels in the relaxed solution; 
(ii) Algorithm~\ref{alg:FindPersistency} is polynomial and provides only persistent labels, whereas the method~\cite{SavchynskyyNIPS2013} has exponential complexity and either finds an optimal solution or gives no information about persistence.

From the practical point of view, both algorithms have different application scenarios:
CombiLP~\cite{SavchynskyyNIPS2013} will only work on sparse graphs, as otherwise the combinatorial part, which one has to solve with exact methods, becomes too big, as the boundary $\partial \mathcal{V}_A$ for $A\subsetneq \mathcal{V}$ grows very quickly then.
Also, even for sparse graphs, the combinatorial part may not grow too big during the application of the algorithm, as otherwise the combinatorial solver will again not be able to cope with it.
Our algorithm does not possess these two disadvantages. From the perspective of running time it does not matter how big the set $\SV\backslash A^t$ becomes during the iterations of Algorithm~\ref{alg:FindPersistency}. 
On the other hand, the subsets of variables to which the method~\cite{SavchynskyyNIPS2013} applies a combinatorial solver to achieve global optimality are often smaller than $\SV\backslash A^t$ in Algorithm~\ref{alg:FindPersistency}, because potentials in CombiLP~\cite{SavchynskyyNIPS2013} remain unchanged in contrast to the perturbation~\eqref{eq:InnerFunctional}. 
Another advantage of the method~\cite{SavchynskyyNIPS2013} is that it needs to solve the (typically) big LP relaxation of the original problem only once, whereas our method does this iteratively, which makes it often slower then CombiLP.

One other possible application scenario which is possible with our method but not with CombiLP~\cite{SavchynskyyNIPS2013} is the following: Assume we want to solve an extremely big inference problem, one that does not fit even into memory. 
To do this, choose a subset $A \subsetneq \mathcal{V}$ of nodes of the graphical model, solve the inference problem on the induced subgraph $G(A)$ with some boundary conditions, and find a partially optimal labeling on it. This is akin to the windowing technique of~\cite{shekhovtsov-14-TR}. By doing so for an overlapping set of subgraphs, one may try to find a labeling for the overall problem on $G$.

The major differences between CombiLP~\cite{SavchynskyyNIPS2013} and our method are summarised in Table~\ref{tab:CombiLPComparison}.
\begin{table}
\centering
\begin{tabular}{|l|c|c|} 
\toprule
\multicolumn{3}{|r|}{\rotatebox{-30}{CombiLP~\cite{SavchynskyyNIPS2013}} \hspace{-1.3cm} \rotatebox{-30}{Our method}} \\
\hline
Dense graphs & \minus & \plus \\
Very large-scale & \minus & \plus \\
Big fractional part of LP solution & \minus & \plus \\
\parbox{5cm}{Relaxed MAP-inference is solved only once} & \plus & \minus \\
\vspace{-0.28cm}&&\\
\parbox{5cm}{Provides a complete solution to Labeling Problem~\eqref{eq:GraphicalModel}} & \plus & \minus \\
\bottomrule
\end{tabular}
\caption{Comparison between our method and CombiLP~\cite{SavchynskyyNIPS2013}.}
\label{tab:CombiLPComparison}
\end{table}

\section{Largest Persistent Labeling}
\label{sec:LargestPersistentLabeling}

Let $A^0\subset \mathcal{V}$ and $\mu^0\in\Lambda_{A^0}$ be defined as in Algorithm~\ref{alg:FindPersistency}.
Subsets $A\subset A^0$ which fulfill the conditions of Corollary~\ref{cor:RelaxedPersistencyCriterion} taken with labelings $\mu^0|_A$ can be partially ordered with respect to inclusion $\subset$. 
In this section we will show that the following holds:
\begin{itemize}
 \item There is a largest set among those, for which there exists a \emph{unique} persistent labeling fufilling the conditions of Corollary~\ref{cor:RelaxedPersistencyCriterion}. 
 \item Algorithm~\ref{alg:FindPersistency} finds this largest set.
\end{itemize}
This will imply that Algorithm~\ref{alg:FindPersistency} cannot be improved upon with regard to the criterion in Corollary~\ref{cor:RelaxedPersistencyCriterion}.

\begin{definition}[Strong Persistency]
A labeling $x^* \in X_{A}$ is called {\em strongly persistent} on $A$, if from
\begin{equation}
x^0 \in \argmin_{x \in X_A} \hat{E}_{A,x^0}(x)\,,
\end{equation}
with $\hat{E}_{A,x^*}$ as in~\eqref{eq:InnerFunctional} follows $x^* = x^0$, i.e. $x^*$ is the \emph{unique} labeling on $A$ such that 
$x^* \in \argmin_{x \in X_A} \hat{E}_{A,x^*}(x)$.
\end{definition}

\begin{lemma}
\label{lemma:StrongPersistencyUniqueGlobalOptimum}
Let $x^* \in X_A$ be strongly persistent. Then for any optimal solution $x$ of~\eqref{eq:GraphicalModel} we have $x^* = x_{|A}$.
\end{lemma}
\begin{proof}
This follows from Inequality~\eqref{eq:PersistencyCriterionLeq1} being strict in this case.
\end{proof}

\begin{theorem}[Largest persistent labeling]
\label{thm:LargestPersistentLabeling}
Let $x^0 \in X_{A^*_{strong}}$ and $A^*_{strong} \subset \mathcal{V}$ be such that 
\begin{equation}
\delta(x^0) \in \argmin_{\mu\in\Lambda_{A^*_{strong}}} \hat{E}_{A^*_{strong},x^0}(\mu)
\end{equation} 
and $x^0$ is the unique such labeling on $A^*_{strong}$.

Then Algorithm~\ref{alg:FindPersistency} finds a persistent labeling on $A^*$ such that $A^*_{strong} \subset A^* \subset \mathcal{V}$, i.e. $A^*$ is a superset of all sets on which strongly persistent labelings identifiable by the criterion of Corollary~\ref{cor:RelaxedPersistencyCriterion} exist.
\end{theorem}
To prove the theorem we need the following technical lemma.

\begin{lemma}
\label{lemma:PersistentLabelingDuringIterations}
Let $A \subset B \subset \mathcal{V}$ be two subsets of $\mathcal{V}$ and $\mu^A \in \Lambda_A$ marginals on $A$ and $x^A \in X_A$ a labeling fulfilling the conditions of Corollary~\ref{cor:RelaxedPersistencyCriterion} uniquely (i.e. $x^A$ is strongly persistent). Let $y^B \in X_{B}$ be a test labeling such that $y^B|_A=x^A$. 

Then for all marginals $\mu^* \in \argmin_{\mu \in \Lambda_B} \hat{E}_{B,y^B}(\mu)$ on $B$ it holds that $\mu^*_v(x^A_v) = 1$ $\forall v \in A$.
\end{lemma}
\begin{proof}
Similar to the proof of Theorem~\ref{thm:PersistencyCriterion}.
Replace $\mathcal{V}$ by $B$.
\end{proof}

\begin{proof}[Proof of Theorem~\ref{thm:LargestPersistentLabeling}]
We will use the notation from Algorithm~\ref{alg:FindPersistency}.
It will be enough to show that for every $\overline{A} \subset \mathcal{V}$ such that there
exists a strongly persistent labeling $\overline{x} \in X_{\overline{A}}$ we have $\overline{A} \subset A^t$ in each iteration of 
Algorithm~\ref{alg:FindPersistency} and furthermore $\overline{x}_v = x^t_v$ for all $v \in \mathcal{V}_{\overline{A}}$.
Hence the union of sets $A'_{strong}$, for which a strongly persistent labeling exists which fulfills the conditions of Corollary~\ref{cor:RelaxedPersistencyCriterion}, is a subset of $A^t$ $\forall t$.
Also by Lemma~\ref{lemma:StrongPersistencyUniqueGlobalOptimum} the associated strongly persistent labelings agree where they overlap, hence we are done.


For $t=0$ apply Lemma~\ref{lemma:PersistentLabelingDuringIterations} with $A := \overline{A}$ and $B:=A^0(=\mathcal{V})$.
Condition
$\overline{x}= y^B|_{\overline A}$ in Lemma~\ref{lemma:PersistentLabelingDuringIterations} is assured by Corollary~\ref{cor:RelaxedPersistencyCriterion}. 
Hence, Lemma \ref{lemma:PersistentLabelingDuringIterations} ensures that for all
$\mu^0 \in \argmin_{\mu \in \Lambda_\mathcal{V}} E(\mu)$ it holds that 
$\mu^0_{v}(\overline{x}_v) = 1$ for all $v \in \overline{A}$. 

Now assume the claim to hold for iteration $t-1$. We need to show that it also holds for
$t$. For this invoke Lemma~\ref{lemma:PersistentLabelingDuringIterations} with 
$A := \overline{A}$,
$B := A^{t-1}$ and
$y^B := x^{t-1}$. The conditions of Lemma~\ref{lemma:PersistentLabelingDuringIterations} hold by assumption on $t-1$. Lemma~\ref{lemma:PersistentLabelingDuringIterations} now ensures that for all $\mu^{t} \in
\argmin_{\mu \in \Lambda_{A^{t-1}}} \hat{E}_{A^{t-1}, x^{t-1}}(\mu)$ 
there holds
$\mu^{t}(x_v^A) = 1$ $\forall v \in A$. 

\end{proof}

From the proof of Theorem~\ref{thm:LargestPersistentLabeling} we can directly conclude the existence of the largest set $A \subset \mathcal{V}$ such that there is a strongly persistent labeling on $A$ identifiably by the criterion in Corollary~\ref{cor:RelaxedPersistencyCriterion}.
\begin{corollary}\label{cor:ExistenceLargestStrictPersistency}
There exists a unique largest set $A^*_{strong}$ with a strongly persistent labeling $x^0 \in A^*_{strong}$ identifiable by the criterion in Corollary~\ref{cor:RelaxedPersistencyCriterion}, i.e.
such that
\begin{equation}
\delta(x^0) \in \argmin_{\mu\in\Lambda_{A^*_{strong}}} \hat{E}_{A^*_{strong},x^0}(\mu)\,,
\end{equation} 
and $x^0$ is the unique such labeling.
\end{corollary}

Also exactly the largest strongly persistent labeling identifiable by Corollary~\ref{cor:RelaxedPersistencyCriterion} can be found under a mild uniqueness assumption.
\begin{corollary}\label{cor:strickt-persistency-optimality}
If there is a \emph{unique} solution of $\min_{\mu \in\Lambda_{A^t}} \hat{E}_{A^t,x^t}(\mu)$ for all $t=0,\ldots$ obtained during the iterations of Algorithm~\ref{alg:FindPersistency}, then Algorithm~\ref{alg:FindPersistency} finds the largest subset of persistent variables identifiable by the sufficient partial optimality criterion in Corollary~\ref{cor:RelaxedPersistencyCriterion}.
\end{corollary}

\begin{remark}
Above we showed that Algorithm~\ref{alg:FindPersistency} will find a persistent labeling which contains the largest strongly persistent one identifiably by Corollary~\ref{cor:RelaxedPersistencyCriterion}. The two may differ when the optimization problems solved in the course of Algorithm~\ref{alg:FindPersistency} have multiple optima.
The simplest example of such a situation occurs if the relaxation $\min_{\mu \in \Lambda_\mathcal{V}} E_{\mathcal{V}}(\mu)$ is tight, but has several integer solutions. 
Any convex combination of these solutions will form a non-integral solution, hence the strongly persistent labeling is defined on a smaller set than any integral solution of $\min_{\mu \in \Lambda_\mathcal{V}} E_{\mathcal{V}}(\mu)$, which is non strongly persistent.
Note however that a labeling obtained by Algorithm~\ref{alg:FindPersistency}, also when it is not strongly persistent, comes from {\em one} globally optimal labeling, i.e. it can be completed to a globally optimal labeling by solving for the remaining variables.
\end{remark}

\section{Optimal Reparametrization}\label{sec:Reparamerisation}
It is well-known~\cite{Schlesinger76} (see also~\cite{ALinearProgrammingApproachToMaxSumWerner}) that representation~\eqref{eq:GraphicalModel} of the energy function is not unique. There are other potentials, which keep the energy of all labelings unchanged. Any such potentials $\theta^{\varphi}$ can be represented as
\begin{align}
  \theta^{\varphi}_{v}(x_v) & :=\theta_{v}(x_v) - \sum_{u\in\nb(v)}\varphi_{v,u}(x_v)\,,\\
  \theta^{\varphi}_{uv}(x_u,x_v) & :=\theta_{uv}(x_u,x_v) + \varphi_{v,u}(x_v) + \varphi_{u,v}(x_u)
\end{align}
with some numbers $\varphi_{u,v}(x_u)$, $uv\in\SE$, $x_u\in X_u$, where $\nb(v):=\{u\in\SV\colon uv\in\SE\}$ denotes the set of nodes adjacent to $v\in\SV$. The vector $\varphi$ with coordinates $\varphi_{u,v}(x_u)$ is called {\em reparametrization}. 

The boundary potentials~\eqref{eq:InnerPotential} and hence the persistency approach described above are dependent on reparametrization. The natural question is existence of {\em an optimal reparametrization}, that is, the one providing the largest persistent set.

The only coordinates of the reparametrization vector $\varphi$, which can potentially influence the solution of the test problem~\eqref{eq:PersistencyMinimizationProblem} are $\varphi_{v,u}(x_v)$, $u\in \partial \SV_A$, $uv\in\partial\SE_A$. Reparametrization $\varphi_{v,u}(x_v)$, $v\in A$ "inside" $A$ does not influence the solution, because it does not change the augmented energy $\hat E_{A,y}$ of any labeling $y$. Similarly, the reparametrization $\varphi_{u,v}(x_u)$, $u,v\notin A$ "outside" $A$ does not influence it, because the optimization is performed over $A$ only.

Considering the reparametrized potentials $\theta^{\varphi}$ and subtracting $\max_{x_v \in X_v} \theta_{uv}(y_u,x_v)$ in~\eqref{eq:InnerPotential} the boundary potentials $\hat{\theta^{\varphi}}_{uv,y_u}(x_u)$ can be equivalently exchanged with 
\begin{equation}
\label{eq:InnerPotential-rewritten}
\left\{ 
\begin{array}{ll}
0, & y_u = x_u\\
\min\limits_{x_v \in X_v} \theta^{\varphi}_{uv}(x_u,x_v) - \max\limits_{x_v \in X_v} \theta^{\varphi}_{uv}(y_u,x_v), & y_u \neq x_u\\
\end{array}\,.
\right.
\end{equation}
It means that the labelings $x$ not coinciding with $y$ on $\partial \mathcal{V}_A$ will be "encouraged" with (typically negative) value $\Delta^{\varphi}_{uv}(x_u):=\min\limits_{x_v \in X_v} \theta^{\varphi}_{uv}(x_u,x_v) - \max\limits_{x_v \in X_v} \theta^{\varphi}_{uv}(y_u,x_v)$. Intuitively clear that the bigger $\Delta^{\varphi}_{uv}(x_u)$ is, the better the proposal labeling $y|_{A}$ comparing to $x|_{A}\neq y|_{A}$ is and hence the greater the found persistent set $A^*$ returned by Algorithm~\ref{alg:FindPersistency} would be. We will prove correctness of this intuition formally, but first let us find {\em the maximal possible} value of  $\Delta^{\varphi}_{uv}(x_u)$  w.r.t. the reparametrization $\varphi$, where we consider as non-zero only coordinates $\varphi_{v,u}(x_v)$, $u\in \partial \mathcal{V}_A$, $uv\in\partial\SE_{A}$, $x_v\in X_v$.

Clearly 
\begin{multline}\label{eq:best-repa-relation}
\Delta^{\varphi}_{uv}(x_u)\le \min_{x_v \in X_v} (\theta^{\varphi}_{uv}(x_u,x_v) - \theta^{\varphi}_{uv}(y_u,x_v))\\
=\min_{x_v \in X_v} (\theta_{uv}(x_u,x_v)+\varphi_{v,u}(x_v) - \theta_{uv}(y_u,x_v)-\varphi_{v,u}(x_v))\\
=\min_{x_v \in X_v} (\theta_{uv}(x_u,x_v) - \theta_{uv}(y_u,x_v))\,,
\end{multline}
hence, the right-hand-side of this inequality does not depend on the reparametrization, whereas the left-hand-side does. There is indeed such a reparametrization that turns the inequality~\eqref{eq:best-repa-relation} into equality and in this way guarantees the largest possible values of $\Delta^{\varphi}_{uv}(x_u)$ for all $x_u$. 
This reparametrization (as we show below it is {\em an optimal} one) is defined as
\begin{equation}\label{eq:optimal-repa}
 \varphi_{u,v}(x_v)= - \theta_{uv}(y_u,x_v)\,,
\end{equation}
which can be seen when plugging~\eqref{eq:optimal-repa} into~\eqref{eq:InnerPotential-rewritten}.

Moreover, since as we mentioned above the reparametrization "outside" an "inside" $A^t$ does not influence the criterion~\eqref{eq:InnerPotential}, we can construct a single, equal for all iterations of Algortihm~\ref{alg:FindPersistency} optimal reparametrization $\psi$ according to the rule~\eqref{eq:optimal-repa} as
\begin{equation}\label{eq:optimal-repa-psi}
 \psi_{u,v}(x_v)=  - \theta_{uv}(y_u,x_v),\ u\in \SV,\ uv\in \SE\,,
 \end{equation}
 where $y$ is arbitrarily extended from $A^0$ to $\SV$.
Now we are ready to formulate our main result related to the reparametrization. 

Let us denote by $\hat E^{\varphi}_{A,y}$ the energy with boundary labeling defined as in Definition~\ref{def:InsidePotentialAndFunctional} w.r.t.\ the potentials $\theta^{\varphi}$. Then for the reparametrization $\psi$ defined as in~\eqref{eq:optimal-repa-psi} there holds
 \begin{lemma}\label{lemma:optimal-repa-property}
  From 
 \begin{equation}\label{equ:part-optimal-y}
 \delta(x^0)\in\arg\min_{\mu\in \Lambda_A}\hat E_{A,x^0}(\mu)
 \end{equation}
 follows  $\delta(x^0)\in\arg\min_{\mu\in \Lambda_A}\hat E^{\psi}_{A,x^0}(\mu)$,
 which means: if $x^0$ satisfies the persistency criterion of Corollary~\ref{cor:RelaxedPersistencyCriterion} w.r.t.\ potentials $\theta$ then it satisfies it w.r.t.\ the reparametrized potentials $\theta^{\psi}$.
 \end{lemma}
  \begin{proof}
   From~\eqref{eq:best-repa-relation} and~\eqref{equ:part-optimal-y} it follows that for all $uv\in\SE_{A}$, $x_u\in X_u$
   there holds
   $\hat\theta^{\psi}_{uv,x^0_x}(x_u) - \hat\theta^{\psi}_{uv,x^0_u}(x^0_u) \ge \hat\theta_{uv,x^0_u}(x_u) - \hat\theta_{uv,x^0_u}(x^0_u)$
   and hence
   \begin{multline}
   \hat E^{\psi}_{A,x^0}(\mu) - \hat E^{\psi}_{A,x^0}(x^0) \stackrel{\eqref{eq:best-repa-relation}}{\ge} \hat E_{A,x^0} (\mu) - \hat E_{A,x^0}(x^0) \ge 0
   \end{multline}
    for all $\mu\in \Lambda_A$. Thus $\hat E^{\psi}_{A,x^0}(x^0) \le \hat E^{\psi}_{A,x^0}(\mu)$, which proves the statement of the lemma. 
 \end{proof}
 \begin{remark}
 Lemma~\ref{lemma:optimal-repa-property} holds for {\em any} polytope containing all integer solutions, i.e.\ $\Lambda_{A}\supseteq \SM_A$ and hence it holds also when $\Lambda_{A}=\SM_A$. In this case it corresponds to the non-relaxed persistency criterion provided by Theorem~\ref{thm:PersistencyCriterion}. 
 \end{remark}

  Let now $A_{x^0}^{\varphi,*}$ be {\em the largest} set containing all strongly persistent variables satisfying Corollary~\ref{cor:RelaxedPersistencyCriterion} w.r.t.\ the reparametrized potentials $\theta^{\varphi}$ and test labeling $y\in X_{\SV}$. Let also $A^{*}_{x^0}$ correspond to the trivial reparametrization  $\varphi\equiv 0$. 

Applying Lemma~\ref{lemma:optimal-repa-property} to the set $A^{*}_{x^0}$ leads to the following
\begin{theorem}\label{thm:best-repa-greates-A}
 For any test labeling $x^0\in X_{\SV}$ there holds ${A^{*}_{x^0}\subset A^{\psi,*}_{x^0}}$.
\end{theorem}
\begin{proof}
Same proof as in Lemma~\ref{lemma:optimal-repa-property} applied to $A^{*}_{x^0}$.
\end{proof}

\begin{remark}
 For Potts models, where 
 $\theta_{uv}(x_u,x_v)=\left\{
 \begin{array}{cl}
  0, & x_u=x_v\\
  \alpha, & x_u\neq x_v
 \end{array}
\right.\,,
 $
 the inequality~\eqref{eq:best-repa-relation} holds as equality also for the trivial reparametrization $\varphi_{v,u}(x_v)=0$ $\forall u,v\in\SV$, $uv\in\SE$, $x_v\in X_v$. For such models Algorithm~\ref{alg:FindPersistency} with the trivial reparametrization delivers the same persistent set as with the optimal one~\eqref{eq:optimal-repa-psi}.
\end{remark}

\section{Optimality of the Method}
Theorem~\ref{thm:LargestPersistentLabeling} proves optimality of Algorithm~\ref{alg:FindPersistency} w.r.t.\ the formulated persistency criterion provided by Theorem~\ref{thm:PersistencyCriterion}. However it does not prove optimality of the method with respect to other possible criteria and hence does not guarantee its superiority over other partial optimality techniques. There is however a recent study~\cite{ExactAndPartialEnergyMinimizationShekhovtsov,ShekhovtsovCVPR2014}, which provides such {\em an optimal relaxed persistency criterion} covering {\em all} existing methods. In what follows we will introduce key notions from~\cite{ShekhovtsovCVPR2014} and show that our persistency criterion coincides with {\em the optimal} one provided in~\cite{ShekhovtsovCVPR2014} for a certain class of persistency methods, those providing only node-persistency, i.e.\ either eliminating all labels except one in a given node or not eliminating any. 




\begin{definition}\label{def:improvingMapping} A mapping $p\colon X_{\SV}\to X_{\SV}$ is called {\em (strictly) improving} if for all $x\in X_{\SV}$ such that $p(x)\neq x$ there holds
$
E_{\SV}(p(x)) \le E_{\SV}(x)$ (resp. $E_{\SV}(p(x)) < E_{\SV}(x)$).
\end{definition}

In what follows we will restrict ourselfs only to {\em idempotent} mappings $p$, i.e. satisfying $p(p(x))=p(x)$.

Following~\cite{ShekhovtsovCVPR2014} we consider only {\em node-wise} maps of the form $p(x)_v = p_v(x_v)$, where $p_v\colon X_v\to  X_v$ are idempotent, i.e.\ $p_v(p_v(x_v))=p_v(x_v)$ for all $x_v\in X_v$. This class is already general enough to include nearly all existing techniques.
 
Improving mappings define persistency due to the following proposition:
\begin{proposition}[Stat.1\cite{ShekhovtsovCVPR2014}]\label{prop:ImprovingMappingProperty}
 Let $p$ be an improving mapping. Then there exists an optimal solution $x$ of~\eqref{eq:GraphicalModel} such that for all $v\in \SV$ from $p_v(i) \neq i$ follows $x_v\neq i$. In case $p$ is strictly improving this holds for {\em any} optimal solution.
\end{proposition}

For an idempotent mapping $p$ a linear mapping $P\colon \R^\SI \to \R^\SI$ satisfying $\delta(p(x)) = P\delta(x)$ for all $x\in X_{\SV}$ is called its {\em linear extension}. A particular linear extension denoted as $[p]$ is defined as follows. For each $p_v$ we define the matrix $P_v \in \R^{ X_v \times  X_v}$ 
by 
$P_{v,ii'} = \left\{
\begin{array}{rl}
1, &  p_v(i') = i\\
0, &  p_v(i') \neq i
\end{array}
\right.
$.
The linear extension $P = [p]$ is given by
\begin{align}\label{equ:PixelLinearExtension}
& (P\mu)_{v} = \sum\nolimits_{i'\in X_v}P_{v, ii'} \mu_{v}(i') = P_v \mu_v;\\
& (P\mu)_{uv} = P_{u} \mu_{uv} P_{v}^{\top}. \nonumber
\end{align}
In what follows we will employ the commonly used representation of energy $E_{\SV}(\mu)$ in a form of an inner product $\lan\theta,\mu\ran$, where vectors of potentials $\theta$ and marginals $\mu$ belong to the vector space $\R^{\SI}$ with the suitably selected dimension $\SI=\sum\limits_{v\in \SV}|X_v|+\sum\limits_{uv\in \SE}|X_{uv}|$.
Denote by $I$ the identity matrix. From Definition~\ref{def:improvingMapping} follows that $p$ is improving iff the value of
\begin{multline}\label{equ:linearStrictImprovingMapping}
\min_{x\in X_{\SV}}\left( E_{\SV}(x)-E_{\SV}(p(x))\right)=\min_{x\in X_{\SV}}\lan \theta,(I-[p])\delta(x)\ran\\
 =\min_{x\in X_{\SV}}\lan(I-[p])^{\top} \theta,\delta(x)\ran
 =\min_{\mu\in\SM_{\SV}}\lan(I-[p])^{\top} \theta,\mu\ran
\end{multline}
is zero. If additionally $p(x)=x$ for all minimizers of~\eqref{equ:linearStrictImprovingMapping} then the mapping $p$ is strictly improving. 
 
Problem~\eqref{equ:linearStrictImprovingMapping} is of the same form as energy minimization~\eqref{eq:GraphicalModel} and is therefore as hard as Problem~\eqref{equ:linearStrictImprovingMapping}.
Its relaxation is obtained by letting $\mu$ to vary in the local polytope $\Lambda_{\SV}\subset \R^\SI$, an outer approximation to $\SM_{\SV}$. 

\begin{definition}\label{def:linearStrictLPImprovingMapping}
An idempotent mapping $p\colon  X_{\SV} \to  X_{\SV}$ is {\em $\Lambda_{\SV}$-improving for potentials $\theta\in\R^{\SI}$} if 
 \begin{equation}\label{equ:linearStrictLPImprovingMapping}
 	\min_{\mu\in \Lambda_{\SV}}\lan(I-[p])^{\top} \theta,\mu\ran=0\,.
 \end{equation}
If additionally $[p]\mu=\mu$ for all minimizers $\mu$ of~\eqref{equ:linearStrictLPImprovingMapping} then $p$ is strictly $\Lambda_{\SV}$-improving.
\end{definition}

Compared to~\eqref{equ:linearStrictImprovingMapping}, only the polytope was changed to $\Lambda_{\SV}\supset\SM_{\SV}$. 
This implies the following simple fact:
\begin{proposition}\label{prop:LambdaImprovingIsImproving}
 If mapping $p$ is (strictly) $\Lambda_{\SV}$-improving then it is (strictly) improving.
\end{proposition}
  
The method presented in this work can be interpreted as considering 
{\em all-to-one} node-wise idempotent mappings~$p$ having the form 
\begin{equation}\label{equ:pixelWiseMapping}
 p_v(i)=\left\{
 \begin{array}{ll}
  y_v, & \mbox{if}\ v\in A \\
  i, & \mbox{if}\ v\notin A
 \end{array}
 \right.
\end{equation}
for a fixed {\em test labeling} $y$. All labels in the nodes $v\in A\subset\SV$ are mapped to $y_v$. 
Among all all-to-one (strictly) $\Lambda_{\SV}$-improving mappings the one with the largest set $A$ will be called {\em maximal}.
  

Corollary~\ref{cor:RelaxedPersistencyCriterion} determines $\Lambda_{\SV}$-improving mappings, as stated by
\begin{lemma}\label{lemma:lambda-improving}
 The relaxed persistency criterion provided by Corollary~\ref{cor:RelaxedPersistencyCriterion} with the reparametrization given by~\eqref{eq:optimal-repa-psi} is equivalent to Definition~\ref{def:linearStrictLPImprovingMapping} with the improving mapping $p$ defined as in~\eqref{equ:pixelWiseMapping} for a given test labeling $y$.
\end{lemma}
\begin{proof}
For future references we write down potentials $\theta^{\psi}$ with $\psi$ defined by~\eqref{eq:optimal-repa-psi} explicitly:
\begin{align}\label{eq:reparametrized-potentials}
\theta^{\psi}_{u}(x_u) & =\theta_{u}(x_u)+\sum_{v\in\nb(u)}\theta_{uv}(x_u,y_v)\,,\\
\theta^{\psi}_{uv}(x_u,x_v) & =\theta_{uv}(x_u,x_v)-\theta_{uv}(x_u,y_v)-\theta_{uv}(y_u,x_v)\,.\nonumber
\end{align}


In what follows we will show that the criteria~\eqref{eq:TractablePartialOptimalityCriterion} and~\eqref{equ:linearStrictLPImprovingMapping} coincide. Both of them represent the local polytope relaxation of specially constructed energy minimization problems. To prove that the relaxations coinside it is sufficient to prove that the non-relaxed energies are equal. 

\myparagraph{Energy of Criterion~\ref{cor:RelaxedPersistencyCriterion}.}
First we write down the non-relaxed test problem~\eqref{eq:PersistencyMinimizationProblem} with potentials $\theta^{\psi}$ as 
\begin{equation}
\arg\min_{x\in X_{\SV}}\sum_{v\in \mathcal{V}}\beta_{v}(x_v)+
\sum_{uv\in \mathcal{E}}\beta_{uv}(x_u,x_v)+
\hspace{-15pt}\sum_{uv\in\partial \mathcal{E}_{A}\colon u \in \partial \mathcal{V}_A}\hspace{-15pt}\hat\theta^{\psi}_{uv,y_u}(x_u)
\end{equation}
with potentials $\beta$ equal to $\theta^{\psi}$ on $A$ and vanishing outside it, i.e.

\begin{equation}\label{eq:beta-unaries}
 \beta_{u}(x_u)=
 \left\{
 \begin{array}{rl}
 \theta_{u}(x_u)+\sum\limits_{v\in\nb(u)}\theta_{uv}(x_u,y_v), & u\in A\\
 0, & u\in \SV\backslash A\\
 \end{array}
 \right.
 \end{equation}

 \begin{multline}\label{eq:beta-pairwise}
 \hspace{-0.25cm}\beta_{uv}(x_u,x_v)=\\
  \hspace{-0.5cm}\left\{
  \begin{array}{rl}
  \theta_{uv}(x_u,x_v)-\theta_{uv}(x_u,y_v)-\theta_{uv}(y_u,x_v), & u,v \in A\\
  0, & \text{otherwise}\,.\\
  \end{array}
  \right.
\end{multline}

Border potentials $\hat\theta^{\psi}$ for $uv\in\SE,\ u\in \SV_A,\ v\in\SV\backslash A$ and $x_u\neq y_u$ read:
\begin{multline}\label{eq:hat-theta-boundary}
 \hat\theta^{\psi}_{uv,y_u}(x_u)=\min_{x_v\in X_v}\theta^{\psi}_{uv}(x_u,x_v)= \\
 =\min_{x_v\in X_v}(\theta_{uv}(x_u,x_v)-\theta_{uv}(x_u,y_v)-\theta_{uv}(y_u,x_v))\\
 =-\theta_{uv}(x_u,y_v)+\min_{x_v\in X_v}(\theta_{uv}(x_u,x_v)-\theta_{uv}(y_u,x_v))\,;
\end{multline}
for $x_u=y_u$:
 \begin{multline}\label{eq:hat-theta-y-u}
  \hat\theta^{\psi}_{uv,y_u}(y_u)=\max_{x_v\in X_v}\theta^{\psi}_{uv}(y_u,x_v) = \\
  = \max_{x_v\in X_v}(\theta_{uv}(y_u,x_v)-\theta_{uv}(y_u,y_v)-\theta_{uv}(y_u,x_v))\\
  = -\theta_{uv}(y_u,y_v) \,.
 \end{multline}
Note that~\eqref{eq:hat-theta-boundary} turns into~\eqref{eq:hat-theta-y-u} when $x_u=y_u$, hence it is sufficient to use only expression~\eqref{eq:hat-theta-boundary}.

\myparagraph{Energy of Definition~\ref{def:linearStrictLPImprovingMapping}.}
The non-relaxed version of condition~\eqref{equ:linearStrictLPImprovingMapping} defining $\Lambda_{\SV}$-improving all-to-one mapping, with the labeling proposal $y$ can be formulated as checking whether
\begin{equation}
y\in\arg\hspace{-1.5pt}\min\limits_{x\in X_{\SV}}\sum_{v\in\SV}\hspace{-1.5pt}\gamma_{v}(x_v)+\sum_{uv\in\mathcal{E}}\hspace{-3pt}\gamma_{uv}(x_u,x_v)+\sum_{u\in\partial\SE_{A}}\hspace{-6pt}\hat\gamma_{uv,y_u}(x_u)
\end{equation}
with potentials $\gamma$ defined as:
\begin{equation}\label{eq:gamma-unaries}
 \gamma_{u}(x_u)=
 \left\{
 \begin{array}{rl}
 \theta_u(x_u)-\theta_u(y_u), & u\in A\\
 0, & u\in \SV\backslash A\\
 \end{array}
 \right.
 \end{equation}
 \begin{multline}\label{eq:gamma-pairwise}
 \gamma_{uv}(x_u,x_v)=\\
  \left\{
  \begin{array}{rl}
  \theta_{uv}(x_u,x_v)-\theta_{uv}(y_u,y_v), & u,v\in A\\
  0, & \text{otherwise}\,.\\
  \end{array}
  \right.
\end{multline}

and the border term
\begin{equation}\label{eq:gamma-boundary}
\hat\gamma_{uv,y_u}(x_u)=\min_{x_v\in X_v}(\theta_{uv}(x_u,x_v)-\theta_{uv}(y_u,x_v))\,. 
\end{equation}

\myparagraph{Equivalency of Energies.}
Comparing~\eqref{eq:gamma-unaries},~\eqref{eq:gamma-pairwise} and~\eqref{eq:gamma-boundary} to~\eqref{eq:beta-unaries},~\eqref{eq:beta-pairwise} and~\eqref{eq:hat-theta-boundary} respectively it can be seen that they can be transformed to each other by several operations, which {\em equally} change energies of all labelings and thus do not influence the criterions provided by Theorem~\ref{thm:PersistencyCriterion} and~\cite[eq.(14)]{ShekhovtsovCVPR2014}. These operations are:
\begin{enumerate}
 \item Subtract $\theta_u(y_u)$ from $\beta_u(x_u)$ for all $u\in\SV_A$, $x_u\in X_u$.
 \item Subtract $\theta_{uv}(y_u,y_v)$ from $\beta_{uv}(x_u,x_v)$ for all $uv\in\SE_A$, $(x_u,x_v)\in X_u\times X_v$.
 \item Reparametrize $\beta$ with the reparametrization vector $\phi$ defined as
 \begin{equation}
  \phi_{u,v}(x_u)=
  \left\{
  \begin{array}{rl}
  -\theta_{uv}(x_u,y_v), & u\in A\\
  0, & u\in\SV\backslash A\,.\\
  \end{array}
  \right.
 \end{equation}
\end{enumerate}
\end{proof}


The following theorem states that our method provably delivers the best results among the methods providing node-persistency:

\begin{theorem}
 Under conditions of Corollary~\ref{cor:strickt-persistency-optimality}, Algorithm~\ref{alg:FindPersistency} with the reparametrizations given by~\eqref{eq:optimal-repa-psi} finds the maximal strict $\Lambda_{\SV}$-improving  all-to-one mapping for a given proposal labeling $x^0$.
\end{theorem}
\begin{proof}
Under condition of Corollary~\ref{cor:strickt-persistency-optimality} (i.e.\ when on each iteration there is a unique solution $\mu^t$) Lemma~\ref{lemma:lambda-improving} guarantees equivalence of our criterion (Corollary~\ref{cor:RelaxedPersistencyCriterion} with reparametrization $\psi$) to Definition~\ref{def:linearStrictLPImprovingMapping} for the strict $\Lambda_{\SV}$-improving  all-to-one mapping. Theorem~\ref{thm:LargestPersistentLabeling} states that Algorithm~\ref{alg:FindPersistency} delivers the largest set $A^*$ satisfying this criterion, which in turn proves the theorem.
\end{proof}


\section{Extensions}\label{sec:HigherOrder}

\myparagraph{Higher Order Models.}
Assume now we are not in the pairwise case anymore but have an energy minimization problem over a \emph{hyper}graph ${G=(\mathcal{V},\mathcal{E})}$ with $\mathcal{E} \subset \mathcal{P}(\mathcal{V})$ a set of subsets of $\mathcal{V}$:
\begin{equation}
  \label{eq:HyperGraphicalModel}
\min_{x \in X_{\SV}} E_{\mathcal{V}}(x) := \sum\limits_{e \in \mathcal{E}} \theta_{e}(x_e)\,.
\end{equation}
All definitions, our persistency criterion and Algorithm~\ref{alg:FindPersistency} admit a straightforward generalization.
Analoguously to Definition~\ref{def:Boundary} define for a subset of nodes $A \subset \mathcal{V}$ the boundary nodes as
\begin{equation}
\partial \mathcal{V}_A := \{ u \in A\colon\exists v \in \mathcal{V}\backslash A, \exists e \in \mathcal{E} \text{ s.t. } u,v \in e \}
\end{equation}
and the boundary edges as
\begin{equation}
\partial \mathcal{E}_A := \{ e \in \mathcal{E}\colon\exists u \in A, \exists v \in \mathcal{V} \backslash A \text{ s.t. } u,v \in e\}\,.
\end{equation}

The equivalent of boundary potential in Definition~\ref{def:InsidePotentialAndFunctional} for $e \in \partial \mathcal{E}_A$ is
\begin{equation}
 \hat{\theta}_{e,y}(x) := \left\{ 
\begin{array}{ll}
 \max\limits_{\tilde{x} \in X_e\colon \tilde{x}_{|A \cap e} = x_{|A \cap e}} \theta_{e}(\tilde{x}), & x_{|A \cap e} = y_{|A \cap e} \\
 \min\limits_{\tilde{x} \in X_e\colon \tilde{x}_{|A \cap e} = x_{|A \cap e}} \theta_{e}(\tilde{x}), & x_{|A \cap e} \neq y_{|A \cap e}
\end{array}\,.
\right.
\end{equation}
Now Theorem~\ref{thm:PersistencyCriterion}, Corollary~\ref{cor:RelaxedPersistencyCriterion} and Algorithm~\ref{alg:FindPersistency} can be directly translated to the higher order case.

\myparagraph{Tighter Relaxations.} 
Essentially, Algorithm~\ref{alg:FindPersistency} can be applied also to tighter relaxations than $\Lambda_A$, e.g. when one includes cycle inequalities~\cite{CycleRelaxationSontagThesis}. 
One merely has to replace the local polytope $\Lambda_A$ for $A \subset \mathcal{V}$ by the tighter feasible convex set:

\begin{proposition}
\label{prop:TighterRelaxations}
Let the polytopes $\tilde\Lambda_A\supseteq\SM_A$ satisfy ${\tilde{\Lambda}_A \subset \Lambda_A}$ $\forall A\subset \mathcal{V}$.
Use $\tilde{\Lambda}_{A^t}$ in place of $\Lambda_{A^t}$ in Algorithm~\ref{alg:FindPersistency} and let $\tilde{A}^*$ be the corresponding persistent set returned by the modified algorithm. 
Let $A^*_{strong} \subset A^*$  be the largest subset of strongly persistent variables identifiable by Corollary~\ref{cor:RelaxedPersistencyCriterion} subject to the relaxations $\tilde{\Lambda}_A$ and $\Lambda_A$.
Then $A^*_{strong} \subset \tilde{A}^*_{strong}$.
\end{proposition}

\begin{remark}
For approximate dual solvers for tighter relaxations like~\cite{TighteningLPRelaxationForMAPUsingMessagePassing,FrustratedCyclesSontag} there are analogues of strict arc-consistency, hence these are also integrally correct algorithms as in Definition~\ref{def:ConsistencyMapping} and we can also use these algorithms in Algorithm~\ref{alg:FindPersistency} with the obvious modifications.
\end{remark}

Optimal reparametrization for tighter relaxations and higher order models is beyond the scope of this paper.

\section{Experiments}
\label{sec:Experiments}

We tested our approach with initial and optimal reparametrizations (described in Section~\ref{sec:Reparamerisation}) on several datasets from different computer vision and machine learning benchmarks, 47 problem instances overall, see Table~\ref{tab:Experiments}. We describe each dataset and the corresponding experiments in detail below.

\myparagraph{Competing methods.} We compared our method to {\bf MQPBO}~\cite{MQPBOShekhovtsovTP,PartialOptimalityInMultiLabelMRFsKohli}, {\bf Kovtun}'s method~\cite{KovtunPartialOptimalLabeling}, Generalized Roof Duality (\textbf{GRD}) by Kahl and Strandmark~\cite{GeneralizedRoofDualityStrandmark}, Fix et al's~\cite{HigherOrderGraphCutFix} and Ishikawa's Higer Order Clique Reduction (\textbf{HOCR})~\cite{TransformationGeneralToFirstOrderIshikawa} algorithms. For the first two methods we used our own implementation, and for the other the freely available code of Strandmark~\cite{GeneralizedRoofDualityImplementationStrandmark}.
We were unable to compare to the method of Windheuser et al.~\cite{GeneralizedRoofDualityWindheuser}, because the authors do not give a description for implementing their method in the higher order case and only provide experimental evaluation for problems with pairwise potentials, where their method coincides with MQPBO~\cite{PartialOptimalityInMultiLabelMRFsKohli}.

\myparagraph{Implementation details.} We employed TRWS as an approximate solver for Algorithm~\ref{alg:FindPersistency} and strong tree agreement as a consistency mapping (see Proposition~\ref{prop:PersistencyAlgorithmWithConsistentSolver}) for most of the pairwise problems. We stop TRWS once it has either arrived at (i) tree-agreement; (ii) a small duality gap of $10^{-5}$; (iii) when number of nodes with tree agreement did not increase over the last $100$ iterations or (iv) overall $1500$ iterations. 
For the higher-order models \texttt{protein-interaction}, \texttt{cell-tracking} and \texttt{geo-surf} we employed CPLEX~\cite{CPLEX} as an exact linear programming solver. 
We have run Algorithm~\ref{alg:FindPersistency} with boundary potentials computed as in~\eqref{eq:InnerPotential} for all problems and with boundary potentials computed with the optimal reparametrization as in~\eqref{eq:InnerPotential-rewritten} for the pairwise problems.

\myparagraph{Datasets and Evaluation.} We give a brief characterization of all datasets and report the obtained total percentage of persistent variables of our and competing methods in Table~\ref{tab:Experiments}. The percentage of partial optimality is computed as follows: Suppose we have found a persistent labeling on set $A \subset \mathcal{V}$. Then the percentage is 
$1 - \frac{\sum_{u\not\in A} \log |X_u|} {\sum_{u\in\V} \log |X_u| }$.
Note that by this formulation we take into account the size of the label space for each node.
For an uniform label space the above formula equals $\frac{\abs{A}}{\abs{\mathcal{V}}}$. The latter measure was used in~\cite{PartialOptimalitySwobodaCVPR2014}.
\begin{remark}
Note that in comparison to our conference paper~\cite{PartialOptimalitySwobodaCVPR2014}, persistency results for some datasets with higher order potentials, which were solved with CPLEX are lower now. This is due to two reasons: First, we weight the size of the label space instead of simply counting the number of variables which are partially optimal. In models with nonuniform label space our method tends to find partial optimality for nodes with small label space, hence the new formula gives a smaller percentage. Second, our original research implementation contained subtle bugs which resulted in a higher number of wrongly assigned partially optimal nodes for these models.
We apologize for reporting incorrect results in the experimental section of~\cite{PartialOptimalitySwobodaCVPR2014}.
\end{remark}

\begin{table*}
 \centering
\begin{tabular}{lggggrrrrrrr}
\toprule
\rotatebox{90}{Experiment} & \#I&  \#L&          \#V&O   & 
\rotatebox{90}{MQPBO~\cite{PartialOptimalityInMultiLabelMRFsKohli}} & 
\rotatebox{90}{Kovtun~\cite{KovtunPartialOptimalLabeling}}  &     
\rotatebox{90}{GRD~\cite{GeneralizedRoofDualityStrandmark}} &   
\rotatebox{90}{Fix~\cite{HigherOrderGraphCutFix}}&
\rotatebox{90}{HOCR~\cite{TransformationGeneralToFirstOrderIshikawa}} &
\rotatebox{90}{Ours original~\cite{PartialOptimalitySwobodaCVPR2014}} &
\rotatebox{90}{Ours optimal} \\
\midrule
\texttt{teddy}                 &   1 &        60 &       168749 & 2 &       0 &$\dagger$&$\dagger$&$\dagger$&$\dagger$&\textbf{0.3820}&\textbf{0.3820}\\
\texttt{venus}                 &   1 &        20 &       166221 & 2 &       0 &$\dagger$&$\dagger$&$\dagger$&$\dagger$&0&0\\
\texttt{family}                &   1 &         5 &       425631 & 2 &  0.0432 &$\dagger$&$\dagger$&$\dagger$&$\dagger$&0.0044&\textbf{0.0611}\\
\texttt{pano}                  &   1 &         7 &       514079 & 2 &  0.1247 &$\dagger$&$\dagger$&$\dagger$&$\dagger$&0.2755&\textbf{0.3893}\\
\texttt{Potts}                 &  12 &  $\leq$12 & $\leq$424720 & 2 &  0.1839 &  0.7475 &$\dagger$&$\dagger$&$\dagger$&\textbf{0.9220}&\textbf{0.9220}\\
\texttt{side-chain}            &  21 & $\leq$483 &   $\leq$1971 & 2 &  0.0247 &$\dagger$&$\dagger$&$\dagger$&$\dagger$& 0.1747&\textbf{0.2558}\\
\parbox{2cm}{\texttt{protein \vspace*{-0.15cm} \\-interaction\vspace*{0.15cm}}}   &   8 &         2 &  $\leq$14440 & 3 &$\dagger$&$\dagger$&  \textbf{0.2603} &  0.2545 &  0.2545 &0.0008&$\dagger$\\
\texttt{cell-tracking}         &   1 &         2 &        41134 & 9 &$\dagger$&$\dagger$&$\dagger$&0.1771 &$\dagger$&\textbf{0.2966}&$\dagger$\\
\texttt{geo-surf}&$\dagger$&$\dagger$&$\dagger$&$\dagger$&$\dagger$&$\dagger$&$\dagger$&$\dagger$&$\dagger$&\textbf{0.0743}&$\dagger$\\

\bottomrule
\end{tabular}
\caption{Percentage of persistent variables obtained by methods~\cite{PartialOptimalityInMultiLabelMRFsKohli},\cite{KovtunPartialOptimalLabeling},\cite{GeneralizedRoofDualityStrandmark},\cite{HigherOrderGraphCutFix},\cite{TransformationGeneralToFirstOrderIshikawa} and our methods with boundary potentials computed as in~\eqref{eq:InnerFunctional} (Ours original) and as in~\eqref{eq:InnerPotential-rewritten} (Ours optimal).
Notation $\dagger$ means inapplicability of the method. The columns \#I,\#L,\#V,O  denote the number of instances, labels, variables and the highest order of potentials respectively. 
}
\label{tab:Experiments} 
\end{table*}

The problem instances \texttt{teddy}, \texttt{venus}, \texttt{family}, \texttt{pano}, \texttt{Potts} and \texttt{geo-surf} were made available by~\cite{Kappes2013Benchmark}, while the datasets \texttt{side-chain} and \texttt{protein-interaction} were made available by~\cite{PIC2011}.

The problem instances \texttt{teddy} and \texttt{venus} come from \textbf{the disparity estimation for stereo vision}~\cite{SzeliskiComparativeStudyMRF}. None of the competing approaches was able to find even a single persistent variable for these datasets, presumably because of the large number of labels, whereas we labeled over one third of them as persistent in \texttt{teddy}, though none in \texttt{venus}.

Instances named \texttt{pano} and \texttt{family} come from the \textbf{photomontage} dataset~\cite{SzeliskiComparativeStudyMRF}. These problems have more complicated pairwise potentials than the disparity estimation problems, but less labels. For both datasets we found significantly more persistent variables than  MQPBO, in particular, we were able to label more than a third of the variables in \texttt{pano}.

We also chose $12$ relatively big energy minimization problems with grid structure and \textbf{Potts} interaction terms. 
The underlying application is a color segmentation problem previously considered in~\cite{PartialOptimalityByPruningPotts}. 
Our general approach reproduces results of~\cite{PartialOptimalityByPruningPotts} for the specific Potts model.

\pgfplotsset{clownfish/.append style={
  xlabel = {Algorithm~\ref{alg:FindPersistency} iterations},
  ylabel = {TRWS iterations},
  xmin = -0.5,
  xmax = 45.5,
  ymin = 0,
  ymax = 1000,
  xtick       = {0, 5, 10, 15, 20, 25, 30, 35, 40, 45},
  xticklabels = {0, 5, 10, 15, 20, 25, 30, 35, 40, 45},
  ytick       = { 10, 100, 1000}, 
  yticklabels = { 10, 100, 1000},
  height=5cm,
  width=\textwidth,
  legend pos = north east
}}
\pgfplotsset{TRWS_iterations/.append style = {color={rgb,255:red,190;green,0;blue,138}, mark options = {solid}, mark=square*}}
\begin{figure*}
\centering
\begin{tikzpicture}[scale=0.7]
\begin{semilogyaxis}[clownfish]
\addplot[mark=x,brown] plot coordinates {
(0,654)
(1,537)
(2,594)
(3,347)
(4,448)
(5,434)
(6,301)
(7,356)
(8,260)
(9,564)
(10,274)
(11,421)
(12,279)
(13,207)
(14,429)
(15,331)
(16,332)
(17,320)
(18,208)
(19,100)
(20,270)
(21,325)
(22,298)
(23,351)
(24,243)
(25,195)
(26,294)
(27,243)
(28,100)
(29,100)
(30,100)
(31,225)
(32,108)
(33,189)
(34,230)
(35,198)
(36,100)
(37,100)
(38,100)
(39,100)
(40,100)
(41,100)
(42,100)
(43,100)

};
\addlegendentry{pfau}

\addplot[mark=x,blue] plot coordinates {
(0,223)
(1,122)
(2,122)
(3,122)
(4,122)
(5,122)
};
\addlegendentry{clownfish}

\addplot[mark=x,red] plot coordinates {
(0,260)
(1,158)
(2,90)
(3,90)
(4,90)
(5,90)
(6,90)
(7,90)
(8,90)
(9,90)
(10,90)
(11,90)
(12,90)
(13,90)
(14,90)
(15,90)
(16,90)
(17,90)
(18,90)
(19,90)
(20,90)
(21,90)
(22,90)
(23,90)
(24,90)
(25,90)
(26,90)
(27,90)
(28,90)
(29,90)
(30,90)
(31,90)
(32,90)
(33,90)
(34,90)
(35,90)
(36,90)
(37,90)
};
\addlegendentry{crops}

\end{semilogyaxis}
\end{tikzpicture}
\caption{Iterations needed by TRWS~\cite{TRWSKolmogorov} in Algorithm~\ref{alg:FindPersistency} for three instances from the \texttt{Potts} dataset. }
\label{fig:TRWSIterations}
\end{figure*}
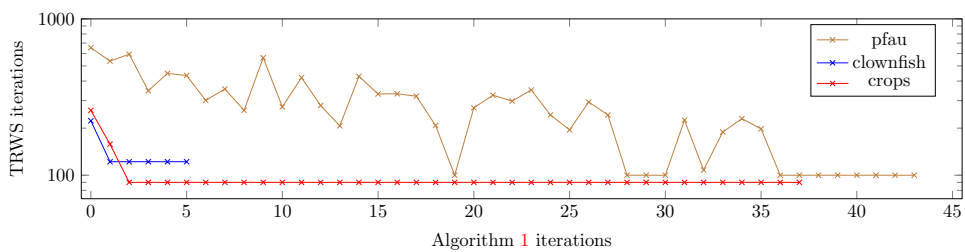

We considered also \texttt{side-chain} prediction problems in {\bf protein folding}~\cite{SideChainPredictionYanover}.  
The datasets consist of pairwise graphical models with $32-1971$ variables and $2-483$ labels. 
The problems with fewer variables are densely connected and have very big label spaces, while the larger ones are less densely connected and have label space up to $81$ variables.


The \texttt{protein interaction} models~\cite{ProteinProteinInteractionNetworkJaimovich} aim to find the subset of proteins, which interact with each other.  
Roof-duality based methods, i.e. Fix et at, GRD, HOCR~\cite{GeneralizedRoofDualityStrandmark,HigherOrderGraphCutFix,TransformationGeneralToFirstOrderIshikawa}
gave around a quarter of persistent labels. This is the only dataset where our methods gives worse results. Note that for higher-order models we do not provide an optimal reparametrization and hence our method is not provably better then the competitors. We consider this as a direction for future work. 

The \textbf{cell tracking} problem consists of a binary higher order graphical model~\cite{CellTrackingByChainGraphKausler}.
Given a sequence of microscopy images of a growing organism, the aim is to find the lineage tree of all cells.
For implementation reasons we were not able to solve \texttt{cell-tracking} dataset with Ishikawa's~\cite{TransformationGeneralToFirstOrderIshikawa} method. However Fix~\cite{HigherOrderGraphCutFix} reports that his method outperforms Ishikawa's method~\cite{TransformationGeneralToFirstOrderIshikawa}.
Other methods are not applicable even theoretically.

Last, we took the higher order multi-label \textbf{geometric surface labeling problems} (denoted as {\tt geo-surf} in Table~\ref{tab:Experiments}) from~\cite{SurfaceLabelingHoiem}. The only instance having an integrality gap has $968$ variables with $7$ labels each and has  ternary terms. Note that MQPBO cannot handle ternary terms, Fix et al's~\cite{HigherOrderGraphCutFix} Ishikawa's~\cite{TransformationGeneralToFirstOrderIshikawa} methods and the generalized roof duality method by Strandmark and Kahl~\cite{GeneralizedRoofDualityStrandmark} cannot handle more than 2 labels. Hence we report our results without comparison.

\myparagraph{Runtime.}  
The runtime of our algorithm mainly depends on the speed of the underlying solver for the local polytope relaxation. 
Currently there seems to be no general rule regarding the runtime of our algorithm, neither in the number of Algorithm~\ref{alg:FindPersistency}-iterations nor in the number of TRWS~\cite{TRWSKolmogorov}-iterations. We show three iteration counts for instances of the \texttt{Potts} dataset in Figure~\ref{fig:TRWSIterations}.


Exemplary pictures comparing the pixels optimally labelled between Kovtuns's method~\cite{KovtunPartialOptimalLabeling} and our method for some Potts-models can be seen in Figure~\ref{fig:KovtunGraphicalComparison}.

\begin{table*}
\label{fig:KovtunGraphicalComparison}
\begin{tabular}{cccccc}
\multicolumn{6}{c}{Kovtun's method~\cite{KovtunPartialOptimalLabeling}} \\
\includegraphics[width=0.145\textwidth]{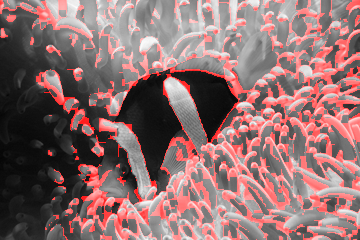} &
\includegraphics[width=0.145\textwidth]{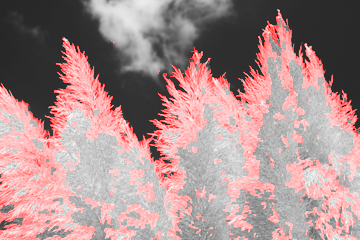} &
\includegraphics[width=0.145\textwidth]{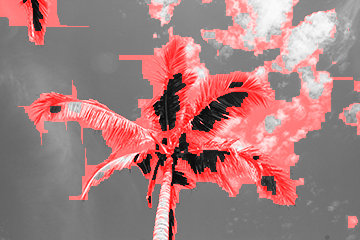} &
\includegraphics[width=0.145\textwidth]{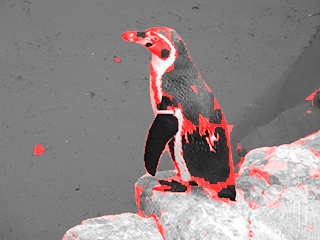} &
\includegraphics[width=0.145\textwidth]{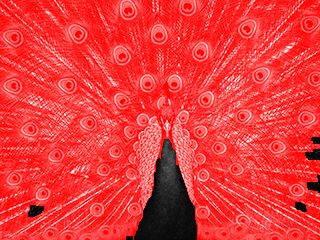} &
\includegraphics[width=0.145\textwidth]{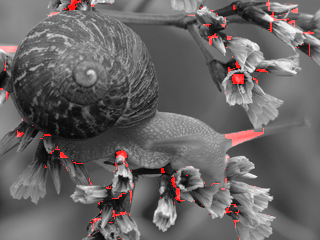}
\\
\multicolumn{6}{c}{Our method} \\
\includegraphics[width=0.145\textwidth]{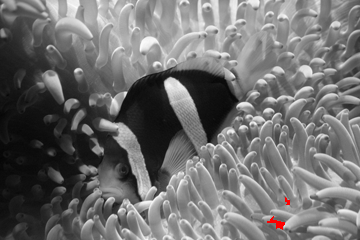} &
\includegraphics[width=0.145\textwidth]{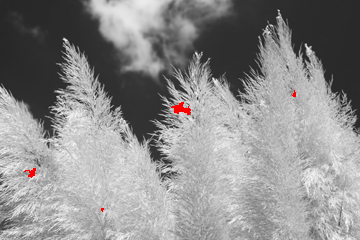} &
\includegraphics[width=0.145\textwidth]{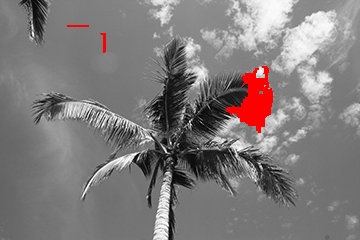} &
\includegraphics[width=0.145\textwidth]{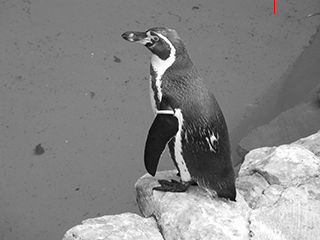} &
\includegraphics[width=0.145\textwidth]{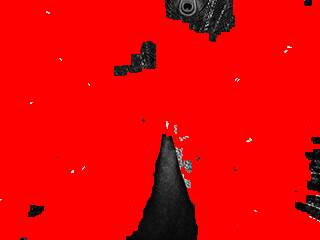} &
\includegraphics[width=0.145\textwidth]{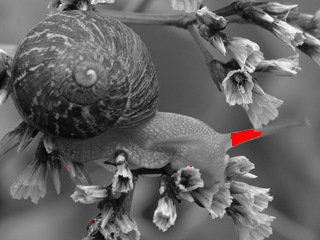}
\\
\end{tabular}
\caption{Comparison between Kovtun's method~\cite{KovtunPartialOptimalLabeling} and our method. The red area denotes pixels which could not be labelled persistently. Contrary to ours the Kovtun's method allows to eliminate separate labels, which is denoted by different intensity of the red color: the more intensive is red, the less labels were eliminated.}
\end{table*}

\section{Conclusion and Outlook}
We have presented a novel method for finding persistent variables for undirected graphical models. 
Empirically it outperforms all tested approaches with respect to the number of persistent variables found on every single dataset.
Our method is general: it can be applied to graphical models of arbitrary order and type of potentials. 
Moreover, there is no fixed choice of convex relaxation for the energy minimization problem and also approximate solvers  for these relaxations can be employed in our approach. 

In the future we plan to significantly speed-up the implementation of our method and consider finer persistency criteria, as done in~\cite{ShekhovtsovCVPR2014}, where the subset-to-one class of persistency conditions was introduced, but no efficient algorithm for finding persistency in this class was proposed. 

\textbf{Acknowledgments.} This work has been supported by the German Research Foundation (DFG) within the program “Spatio-/Temporal Graphical Models and Applications in Image Analysis”, grant GRK 1653. 

{\small
\bibliographystyle{plain}
\bibliography{literatur}
}

\end{document}


\title{Appendix to ``Partial Optimality by Pruning for MAP-Inference with General Graphical Models''}

\author{Paul Swoboda,
	Alexander Shekhovtsov,
        J\"org~Hendrik~Kappes,
        Christoph~Schn\"orr,
        and~Bogdan~Savchynskyy}

\maketitle

\section{Detailed experimental Evaluation for Our Algorithms}

\begin{table*}
\footnotesize
 \centering
\begin{longtable}{llgggrrr|rrr}
\toprule
Experiment & instance & \#L & \#V & O & 
Ours original & \#iter & time (s) &
Ours invariant & \#iter & time (s)\\
\midrule
\multirow{2}{2cm}{\texttt{mrf- \vspace*{-0.15cm} \\stereo\vspace*{0.15cm}}}
& teddy &60&168750&2&0.381268&10539&3590& 0.381961&10539&3571\\
& venus &20&166222&2&0&13202&13945& 0&13202&13956\\
\hline
\multirow{2}{2cm}{\texttt{mrf- \vspace*{-0.15cm} \\photomontage\vspace*{0.15cm}}}
& family &5&425632&2&0.0440991&16604&12278& 0.0450718&16604&9231 \\
& pano  &7&514080&2&0.275465&12705&& 0.389325&12705&17366 \\
\hline
\texttt{Potts}
& clownfish &12&86400&2& 0.999711 &833&152&  0.999711&833&152 \\
& crops &12&86400&2& 0.989988 &3658&668& 0.989988&3658&3658\\
& fourcolors &4&65536&2&0.998734&358&30& 0.998734&358&358\\
& lake &12&86400&2&1&95&14& 1&95&14\\
& palm &12&86400&2&0.981991&5131&830& 0.981991&5131&830\\
& penguin &8&76800&2&1&152&16& 1&152&16\\
& pfau &12&76800&2&0.10431&11665&786& 0.10431&11665&786\\
& snail &3&76800&2&0.999792&581&45& 0.999792&581&45\\
& strawberry &12&76800&2&0.993125&1980&312& 0.993125&1980&312\\
& colseg-cow3 &3&41720&2&0.999535&4664&6292& 0.999535&4664&6292\\
& colseg-cow4 &4&414720&2&0.998006&7080&6994& 0.998006&7080&6994\\
& colseg-garden4 &4&21000&2&0.999095&289&34& 0.999095&289&34\\
\hline
\texttt{side-chain} 
&1CKK&2-445&35&2&0&473&47& 0&471&46\\
&1CM1&2-350&37&2&0&600&51& 0&471&41\\
&1SY9&2-503&33&2&0&574&66& 0.191997&747&77\\ 
&2BBN&2-404&37&2&0&388&53& 0&412&55\\
&2BCX&2-473&39&2&0&446&55& 0&407&51\\
&2BE6&2-470&40&2&0&593&40& 0&490&34\\
&2F3Y&2-439&35&2&0&458&41& 0&407&35\\
&2FOT&2-374&35&2&0&620&61& 0&526&53\\
&2HQW&2-459&36&2&0&503&43& 0&507&42\\
&2O60&2-496&38&2&0&721&93& 0&597&77\\
&3BXL&2-380&36&2&0&443&44& 0&427&43\\
&pdb1b25&2-81&1972&2&0.189202&1681&118& 0.293441&2532&160\\
&pdb1d2e&2-81&1328&2&0.556895&2965&83& 0.920537&654&34\\
&pdb1fmj&2-81&614&2& 0.161837& 863&16&0.263551&1122&20\\
&pdb1i24&2-81&337&2&1&115&1& 1&115&1\\
&pdb1iqc&2-81&1040&2&0.404958&2227&35&0.840613&781&26\\
&pdb1jmx&2-81&739&2 &0.352647&1286&20& 0.460479&1729&30\\
&pdb1kgn&2-81&1060&2&0.131970&2092&53& 0.196760&2896&68\\
&pdb1kwh&2-81&424&2&0.254326&  609&8& 0.323246&716&10\\
&pdb1m3y&2-81&1364&2&0.310822&1335&28&0.490703 &1246&28\\
&pdb1qks&2-81&926&2&0.306292&1315 &28& 0.391010&1687&38\\

\hline
\multirow{2}{2cm}{\texttt{protein- \vspace*{-0.15cm} \\interaction\vspace*{0.15cm}}}
&didNotconverge1&2&14306&3&0.000280&$\dagger$&210&$\dagger$ &$\dagger$  &$\dagger$  \\
&didNotconverge2&2&14441&3&0.004363&$\dagger$&408&$\dagger$ &$\dagger$ &$\dagger$  \\
&didNotconverge3&2&14306&3&0.000280&$\dagger$&211&$\dagger$ &$\dagger$ &$\dagger$  \\
&didNotconverge4&2&14306&3&0.000350&$\dagger$&578&$\dagger$ &$\dagger$ &$\dagger$  \\
&didNotconverge5&2&14258&3&0.000140&$\dagger$&77&$\dagger$ &$\dagger$ &$\dagger$  \\
&didNotconverge6&2&14258&3&0.000140&$\dagger$&174&$\dagger$ &$\dagger$ &$\dagger$  \\
&didNotconverge7&2&14352&3&0.000279&$\dagger$&209&$\dagger$ &$\dagger$ &$\dagger$  \\
&didNotconverge8&2&14352&3&0.000697&$\dagger$&504&$\dagger$ &$\dagger$ &$\dagger$  \\
\hline
\texttt{cell-tracking} 
& ogm\_model&2&41134&9&0.296592&$\dagger$&212&$\dagger$ &$\dagger$ &$\dagger$  \\
\hline
\texttt{geo-surf}
& gm166 &7&969&3&0.074303&$\dagger$&23&$\dagger$ &$\dagger$ &$\dagger$  \\
\bottomrule
\end{longtable}
\caption{Percentage of persistent variables obtained by our methods with boundary potentials computed as in~\eqref{eq:InnerFunctional} (Ours original) and as in~\eqref{eq:InnerPotential-rewritten} (Ours invariant).
Notation $\dagger$ means inapplicability of the method. The columns \#I,\#L,\#V,O  denote the number of instances, labels, variables and the highest order of potentials respectively. The columns \#iter denotes the total number of iterations needed by TRWS for solving the MAP-inference problems in Algorithm~\ref{alg:FindPersistency}.
}
\label{tab:ExperimentsDetailed} 
\end{table*}